\documentclass{article}
\usepackage[nonatbib, final]{water}

\usepackage{booktabs} 
\newtheorem{definition}{Definition}
\newtheorem{proposition}{Proposition}
\newtheorem{proof}{Proof}

\label{usepackage}
\usepackage{algorithm, algpseudocode}
\usepackage{graphicx}
\usepackage[mathscr]{euscript}
\usepackage{epstopdf}
\usepackage{diagbox}
\usepackage{setspace}
\usepackage{amsmath, mdwlist, amssymb, booktabs, ctable, multirow, tabularx, mathtools}
\newcommand{\equationsize}{\footnotesize}

\usepackage{amsmath}
\usepackage{amssymb}

\usepackage{microtype}
\DisableLigatures[f]{encoding = *, family = * }

\usepackage{setspace} 

\usepackage[labelfont=bf,labelsep=period,justification=raggedright]{caption}

\title{Water Disaggregation via Shape Features based Bayesian Discriminative Sparse Coding}

\author{
  Bingsheng~Wang*\thanks{These two authors contributed equally} \\
  Google Corporation\\
  Williamsburg, CA 22315 \\
  \texttt{claren89@vt.edu} \\
   	 \And
 Xuchao~Zhang* \\
  Department of Computer Science\\
  Virginia Tech\\
  Falls Church, VA 22043 \\
  \texttt{xuczhang@vt.edu} \\
     \And
 Chang-Tien~Lu \\
  Department of Computer Science\\
  Virginia Tech\\
  Falls Church, VA 22043 \\
  \texttt{ctlu@vt.edu} \\
     \And
 Feng~Chen \\
  Department of Computer Science\\
  University at Albany-SUNY\\
  Albany, New York \\
  \texttt{xuczhang@vt.edu} \\
}

\begin{document}
\maketitle

%
%
%
%

\begin{abstract}
As the issue of freshwater shortage is increasing daily, it is critical to take effective measures for water conservation. According to previous studies, device level consumption could lead to significant freshwater conservation. Existing water disaggregation methods focus on learning the signatures for appliances; however, they are lack of the mechanism to accurately discriminate parallel appliances' consumption. In this paper, we propose a Bayesian Discriminative Sparse Coding model using Laplace Prior (BDSC-LP) to extensively enhance the disaggregation performance. To derive discriminative basis functions, shape features are presented to describe the low-sampling-rate water consumption patterns. A Gibbs sampling based inference method is designed to extend the discriminative capability of the disaggregation dictionaries. Extensive experiments were performed to validate the effectiveness of the proposed model using both real-world and synthetic datasets.
\end{abstract}

%
%
%

%
%



\section{Introduction}\label{sec:introduction}
The scarcity of potable water is one of the most critical smart-city challenges \cite{huang2016crowdsourcing, zhang2017spatiotemporal, zhang2017traces} facing the world. 
The statistics shown in Nature 2010~\cite{Gilbert10} states that about 80\% of the world's population lives in short of potable water. 
Furthermore, according to the California Department of Water Resources, without more water supplies by 2020, the
region will suffer a deficiency nearly as much as the total amount consumed today~\cite{BENCH-DELVE11}. At the global level, the
existing freshwater is only enough to extend out as much as 60 or 70 years~\cite{Gleick01}. 
Urban water consumption contributes to 50\%$\sim$80\% of public water supply systems and 26\% of whole usage in the US~\cite{Vickers01}.
Several measures have been taken to mitigate the problem of water shortage; 
water conservation is one concrete and fundamental task, where data mining and machine learning can play an important role.

Previous studies have shown that device-level water usage information is crucial for establishing effective conservation strategies~\cite{Fischer08,froehlich2010design,Froehlich12}. 
Water disaggregation refers to the process of separating aggregated smart meter readings into the consumption of
its component appliances, such as Toilet, Shower and Washer.
This paper specifically considers the task of disaggregating residential water consumption targeting for conservation. 
Recently, water disaggregation has become an important topic to explore solutions for water conservation. 
Most previous studies focus on sensing the open/close pressure waves of devices to identify the signatures for separation. 
These methods are capable of achieving more than 90\% in accuracy measure. However, they depend on high sample rate (typically 1 kHz) data sources to analyze individual device features. 
The widely deployed smart meters only produce low sample rate (as low as 1/900 Hz) readings to ensure reliable data transmission, and Figure~\ref{fig:example_consumption} shows an example of real data observation sequences.
It is critical to design an effective algorithm to disaggregate low-sampling-rate smart meter readings.
\begin{figure}[htb!]
\centering
\includegraphics[width=1\textwidth]{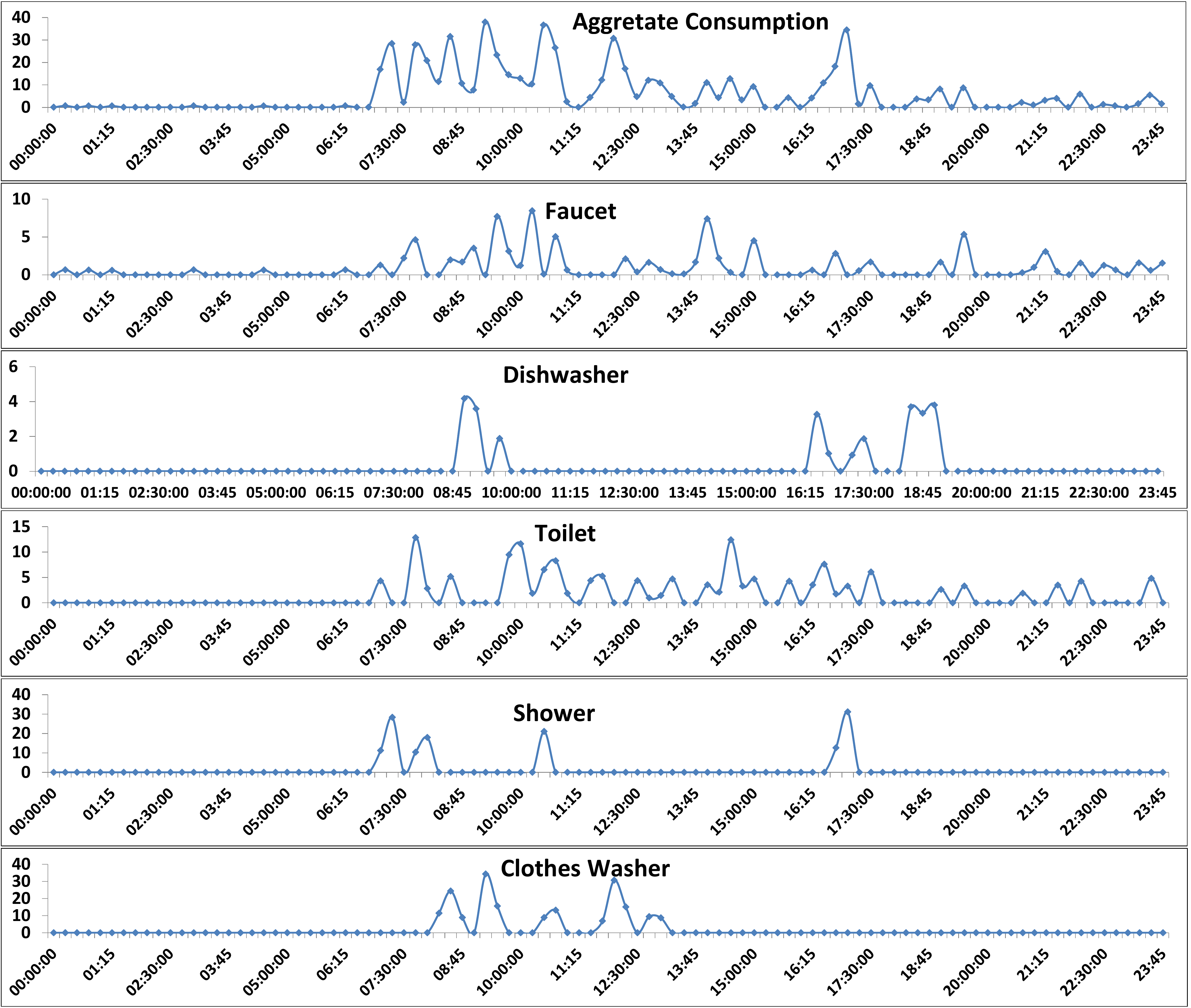}
\vspace{0.0pc}
\caption{Example real data observation sequences in one day}
\vspace{-0.0pc}
\label{fig:example_consumption}
\end{figure}
The critical challenge is to cope with the issues caused by low-sampling-rate data sources. 
Most existing deployed smart meters report one reading per 15 minutes, and it is impossible to identify open/close signatures for appliances like what have done in~\cite{Larson12,Froehlich:11}. 
This requires us to find appliances' signatures from the low-sampling-rate data, and then perform disaggregation.
A Bayesian sparse coding model is proposed to learn the dictionaries for discriminating devices' consumption.
After implementing a formal process to abstract the shape features for each device, the basis functions are initialized using the invariant features, and smoothed to adapt to the variances of label training data.
The sparseness is guaranteed by the Laplace prior distribution over the coefficients.
With the fixture level consumption data, a Bayesian sparse coding model can be learned for each appliance.
By combining the trained dictionaries, the objective function of sparse coding can be further minimized with respect to the aggregated data and this can help to enhance the discriminative capability of the disaggregation dictionaries.
A Gibbs sampling based method is applied to perform inference on the proposed model and the predictive density is evaluated.
In summary, the contributions of this paper are as follows:
\begin{itemize} 
	\item \textbf{Analyses and formalizations of shape features for smart meter readings:}
	Rigorous analyses and definitions of shape features are presented by exploring the prior knowledge with respect to individual device consumption patterns. 
	We use the results of the analyses to guide the learning of basis functions, and show how it can help improve the disaggregation performance.
	
	\item \textbf{Design of a Bayesian discriminative sparse coding model:}
	A Bayesian sparse coding with Laplace prior is learned for each device, and then we combine these trained models together to achieve the disaggregation dictionaries.
	The discriminative capability of the disaggregation dictionaries is improved by adapting the bases to the aggregated data.
	
	\item \textbf{Development of effective inference methods:}
	The Gamma prior over the noise's precision and the Laplace prior over the coefficients make the models to be hard for learning.
	To solve this problem, a Gibbs sampling based algorithm is presented for the inference over the Bayesian sparse coding models.

	\item \textbf{Extensive experiments for illustrating the effectiveness:}
	The effectiveness of the proposed model was validated  with extensive experiments based on both real and synthetic datasets, and the experimental results showed that our model outperformed the baselines.
\end{itemize}

The rest of this paper is organized as follows:
Section~\ref{sec:bk_related} introduces the background and the surveys related work on water disaggregation.
Section~\ref{sec:BDSC_SF} describes the shape features based Bayesian discriminative sparse coding model.
Section~\ref{sec:gs_infer} provides the algorithms for inference and parameters estimation.
The effectiveness of the proposed model is illustrated with extensive experiments in Section~\ref{sec:experiment}.
Finally, Section~\ref{sec:conclusion} presents our conclusions.

\section{Background and Related Work} \label{sec:bk_related}
\subsection{Notations and Concepts}\label{subsec:notation_concepts}
Suppose there is a total of $D$ devices, such as Toilet and Shower.
For each device $d=1, 2, \cdots, D$, $\mathbf{Y}^{(d)}\in\mathbb{R}^{N\times P}$ is used to denote its consumption matrix,
where $N$ is the number of intervals in one day and $P$ is the number of days.
The $p^{\text{th}}$ day's consumption of device $d$ is denoted as $\mathbf{y}^{(d)}_{\cdot,p}$.
The water usage of device $d$ for interval $i$ of day $p$ is denoted as $y^{(d)}_{i, p}$.
$\bar{\mathbf{Y}}$ is used to indicate the aggregated water consumption over all devices: $\mathbf{\bar{Y}} = \sum_{d=1}^D \mathbf{Y}^{(d)}$. 
The $p^{\text{th}}$ column of $\bar{\mathbf{Y}}$ holds the aggregated consumption of the $p^{\text{th}}$ day for a given household. 
The $i^{\text{th}}$ element of $\bar{\mathbf{y}}_{\cdot, p}$, denoted as $\bar{y}_{i,p} = \sum_{d=1}^D y_{i, p}^{(d)}$, is the aggregated consumption at interval $i$ in day $p$.
During the training course, we have the individual device's consumption data, $\mathbf{Y}^{(1)}, \mathbf{Y}^{(2)}, \cdots, \mathbf{Y}^{(D)}$, while during the testing course, only the aggregated data $\mathbf{\bar{Y}}$ is available, with the goal being to separate it into $\mathbf{\hat{Y}}^{(1)}, \mathbf{\hat{Y}}^{(2)}, \cdots, \mathbf{\hat{Y}}^{(D)}$.

\subsection{Disaggregation via Discriminative Sparse Coding}
The basic idea of discriminative sparse coding~\cite{Kolter10} is to employ the regularized disaggregation error as the objective function in place of using the default non-negative sparse coding objective:
\begin{equation}\label{eq:dsc_rde}
E_{reg} = \sum_{d=1}^D \frac{1}{2} \left\| \mathbf{Y}^{(d)} - \mathbf{H}^{(d)}\hat{\mathbf{A}}^{(d)} \right\|_F^2 + \lambda \sum_{d, q, r}(\hat{\mathbf{X}}^{(d)})_{qr}
\end{equation}
where $\hat{\mathbf{X}}^{(1:D)}$ is achieved by optimizing the traditional sparse coding model. 
Minimizing $E_{reg}$ is likely to achieve much better basis functions than optimizing the conventional sparse coding model
for separating the aggregated signal.
The best possible value of $\hat{\mathbf{X}}^{(d)}$ can be achieved by
\begin{equation}\label{eq:dsc_bestX}
\tilde{\mathbf{X}}^{(d)} = \underset{ \mathbf{X}^{(d)}\geq 0 }{\operatorname{argmin}} \; \frac{1}{2}\left\| \mathbf{Y}^{(d)} - \mathbf{H}^{(d)}\mathbf{A}^{(d)} \right\|_F^2 + \lambda \sum_{q, r}(\mathbf{X}^{(d)})_{qr}
\end{equation}
It is obvious that the coefficients achieved by optimizing Eq.~\eqref{eq:dsc_bestX} are the same as the activations obtained after iteratively optimizing the non-negative sparse coding objective. 
As a result, the discriminative dictionary $\tilde{\mathbf{H}}^{(1:D)}$ can be learned by minimizing~\eqref{eq:dsc_rde} while making the activations as close to $\tilde{\mathbf{X}}^{(1:D)}$ as possible.
Since the change of bases $\mathbf{H}^{(1:D)}$ for optimizing~\eqref{eq:dsc_rde} would also cause the resulting optimal coefficients to be changed,
the learned discriminative basis functions (i.e., $\tilde{\mathbf{H}}^{(1:D)}$) would be different from the reconstruction bases (i.e., $\mathbf{H}^{(1:D)}$).
Formally, the discriminative dictionary can be learned by optimizing the augmented regularized disaggregation error objective:
\begin{equation} \label{eq:dsc_arde}
\begin{split}
&\tilde{E}_{reg}\left( \mathbf{Y}^{(1:D)}, \mathbf{H}^{(1:D)}, \tilde{\mathbf{H}}^{(1:D)} \right)
\equiv \sum_{d=1}^D \left( \frac{1}{2}\left\| \mathbf{Y}^{(d)} - \mathbf{H}^{(d)}\mathbf{\hat{X}}^{(d)} \right\|_F^2 + \lambda \sum_{q, r}(\mathbf{\hat{X}}^{(d)})_{qr} \right)
\\
&\qquad\qquad  \text{subject to } \hat{\mathbf{X}}^{(1:D)} = \underset{ \mathbf{X}^{(d)}\geq 0 }{\operatorname{argmin}} \; \frac{1}{2}\left\| \mathbf{Y}^{(d)} - \mathbf{\tilde{H}}^{(d)}\mathbf{X}^{(d)} \right\|_F^2 + \lambda \sum_{q, r}(\mathbf{X}^{(d)})_{qr}
\end{split}
\end{equation}
where $\mathbf{\tilde{H}}^{(1:D)}=[\mathbf{\tilde{H}}^{(1)}, \cdots, \mathbf{\tilde{H}}^{(D)}]$.
$\mathbf{H}^{(1:D)}$ are the reconstruction bases learned from the traditional sparse coding model, 
while $\mathbf{\tilde{H}}^{(1:D)}$ are the discriminative bases optimized by moving $\mathbf{\hat{X}}^{(1:D)}$ as close to $\tilde{\mathbf{X}}^{(1:D)}$ as possible.
It is important to note, however, that although the discriminative sparse coding model has been developed to optimize the bases and thus decrease the overall disaggregation error, it is lack of the mechanism to learn the shape features from the low-sampling-rate data for deriving accurate disaggregation results.

\subsection{Related Work}\label{subsec:related}
Recently, due to the fast-growing amount of data \cite{xu2016fatman, shi2017open, xu2017scaling, Gao2018}, machine learning based approaches \cite{ijcai2017-480, zhang2018distributed, wu2018neural} have been widely applied in many real-world challenges \cite{shi2017proje, zhang2017online, shi2016fact} in smart-city related areas \cite{wu2017uapd, wu2018will}.
Thanks to the increasing deployment of smart meters in many countries, water disaggregation is emerging as an interesting new research direction in urban computing \cite{huang2017scalable, zhang2016storytelling, zhang2017traces}.
Pressure-based sensors have been designed for installation on water fixtures to help identify activity and estimate the corresponding consumption for individual household devices~\cite{Froehlich09,Froehlich:11,Larson12}.
By utilizing both occupancy sensors and whole house water flow meter data,~\cite{Vijay:11} categorized the aggregated consumption at the fine-grained device level.
Although such methods are capable of achieving about 90\% accuracy, they depend on high-sampling-rate sensing data (as high as 1 Khz) to capture the characteristic open/close signatures of devices.
A HMM (Hidden Markov Model) based approach was developed in~\cite{Chen11} for separating low-sampling-rate (1/900 Hz) data, while 
\cite{Nguyen:13,Nguyen:131} proposed a hybrid combination of HMM and DTW (Dynamic Time Warping) to automate the categorisation of residential water end use events and estimate devices' consumption.
However, a HMM based structure inherently restricts its ability to infer consumption for parallel devices.
A deep sparse coding based model was presented in~\cite{Dong13} that fully utilizes the limited label data and performs disaggregation in a sequential manner,
however, the model may be sensitive to the disaggregation structure and the learning process is of high computational complexity when seeking the optimal architecture for disaggregation.

There is a lack of models designed for disaggregating low-sampling-rate water consumption. 
The existing HMM based method analyses the activities with interval based consumption; however, it has limited ability to estimate the consumption for parallel devices due to its inherent serial structure. 
Existing sparse coding based methods \cite{Goodfellow:12, huang2017you} are lack of the mechanism to capture features for better learning disaggregation dictionaries, limiting their capabilities for estimating device level consumption from aggregated data.

\section{Bayesian Discriminative Sparse Coding with Shape Features} \label{sec:BDSC_SF}
We can now formulate the shape features, and perform basis functions smoothness based on prior knowledge and exploration of the low-sampling-rate data, as described below in Section~\ref{subsec:shape}.
Section~\ref{subsec:scg} goes on to describe the sparse coding with Laplace prior to model the generative process, and the 
discriminative model is developed in Section~\ref{subsec:bddm}.

\subsection{Shape Features} \label{subsec:shape}
The domain/prior knowledge suggests that the duration and consumption trend of water fixtures (corresponding to the human activities related to water consumption) are distinct across devices. 
For example, the average duration of Toilet is around 1$\sim$3 minutes, while that of shower is typically between 3$\sim$20 minutes.
This domain knowledge can be incorporated into the sparse coding model to help it learn discriminative dictionaries.
\textbf{Span} is formalized to capture the duration feature and \textbf{First-order Relation} is defined to capture the variations of time series.
\textbf{Consumption Mapping} and \textbf{Shape Features} are then defined to extract devices' consumption dictionaries, and finally \textbf{Basis Smoothness} is introduced to fill the insufficient variances of the learned dictionaries for deriving sparse coefficients.
\begin{definition}[Span] \label{def:span}
For any device $d$, the span of $d$, denoted by $S_d$, is defined as the enumeration of all possible operating durations measured as the number of intervals.
\end{definition}
Taking Toilet as an example, there are two possible time durations, 1 or 2. 
With~\textbf{Definition~\ref{def:span}}, the span for toilet is thus: $S=\{1, 2\}$.
Since water consumption is continuous, it is generally difficult to define trend features. 
For instance, although its duration is limited to only 2 intervals, Toilet might contain infinite combinations as long as the sum is within a certain range (1$\sim$5 gallons~\cite{Roberts04}):
\begin{equation} \label{eq:example_toilet} \equationsize
\left\{ 
	\left[
	\begin{array}{c}
	0.7  \\
	0.8\end{array}
	\right],
	\left[
	\begin{array}{c}
	1.0  \\
	0.5\end{array}
	\right],
	\left[
	\begin{array}{c}
	0.5  \\
	1.0\end{array}
	\right],
	\left[
	\begin{array}{c}
	0.8  \\
	2.5\end{array}
	\right],
	\left[
	\begin{array}{c}
	3.2  \\
	1.7\end{array}
	\right],
	\left[
	\begin{array}{c}
	0.8  \\
	0.8\end{array}
	\right],
	\left[
	\begin{array}{c}
	1.0  \\
	1.0\end{array}
	\right]\cdots
\right\}
\end{equation}
where each vector shows one possible Toilet consumption value distributed over 2 intervals (for simplification, we show here the 2 of the 96 possible intervals in one day). 
The infinite number of possibilities causes the problem to be hard for learning. 
Inspecting the data shown in~\eqref{eq:example_toilet}, we could intuitively discern a pattern: numerical relationships exist (larger than, equal to, or less than) between the values in these two intervals. 
This indicates that the maximum number of consumption trends over $z$ intervals is $3^{z(z-1)/2}$ (where $z$ is any integer bigger than 0), which grows rapidly with the number of intervals. 
To reduce the complexity while capturing the variances, we therefore propose a First-order Relation to approximate the consumption characteristics.
\begin{definition} [First-order Relation]\label{def:FoR}
Given any time
series consisting of $\mathscr{T}$ real values $\mathscr{V}_1, \cdots, \mathscr{V}_{\mathscr{T}}$, where $\mathscr{V}_{\text{min}}$ and $\mathscr{V}_{\text{max}}$ respectively denote the minimum and maximum values of these $\mathscr{T}$ real values,
the First-order Relation of this time series is defined as
\begin{equation} \label{eq:def_FoR} \equationsize
\begin{split}
& \mathcal{F}(1) = 
\begin{cases}
		1  &, \; \mathscr{T} = 1 \\
		1  &, \; \mathscr{V}_1 > \mathscr{V}_2, \mathscr{T}>1\\
		1 &,  \; \mathscr{V}_1 = \mathscr{V}_2 = \mathscr{V}_{\text{max}}, \mathscr{T}>1\\
		0 &, \; \text{otherwise}
	\end{cases}
,
\mathcal{F}(t) = 
\begin{cases}
		1  &, \; \mathscr{V}_t > \mathscr{V}_{t-1}, \mathscr{T}>1 \\
		0  &, \; \mathscr{V}_t < \mathscr{V}_{t-1}, \mathscr{T}>1\\
		\mathcal{F}(t-1) &, \;  \mathscr{V}_t = \mathscr{V}_{t-1}, \mathscr{T}>1\\
	\end{cases}
\\
& \qquad\qquad \text{where}\;t=2, \cdots, \mathscr{T}.
\end{split}
\end{equation}
\end{definition}
For example, applying Eq.~\eqref{eq:def_FoR} reduces the first-order relations\footnote{For illustration purposes, we have not normalized the basis functions.} of vectors in Eq.~\eqref{eq:example_toilet} to:
\begin{equation}\label{eq:example_toilet_approx} \equationsize
\left\{ 
	\left[
	\begin{array}{c}
	0  \\
	1\end{array}
	\right],
	\left[
	\begin{array}{c}
	1  \\
	0\end{array}
	\right],
	\left[
	\begin{array}{c}
	0  \\
	1\end{array}
	\right],
	\left[
	\begin{array}{c}
	0  \\
	1\end{array}
	\right],
	\left[
	\begin{array}{c}
	1  \\
	0\end{array}
	\right],
	\left[
	\begin{array}{c}
	1  \\
	1\end{array}
	\right],
	\left[
	\begin{array}{c}
	1  \\
	1\end{array}
	\right]\cdots
\right\}
\end{equation} 
We can now define the mapping schema to convert infinite continuous consumption values into their corresponding first-order relations:
\begin{definition}[Consumption Mapping] \label{def:con_mapping}
Given any device $d$ and its corresponding span $S_d$. 
For $\forall \mathbf{y}^{(d)}_p\in\mathbf{Y}^{(d)}$, $p=1, 2, \cdots, P$, consider any possible time series combination $\mathcal{C}^{(d)}_{p,r}$ of non-zero values in $\mathbf{y}^{(d)}_p$ while holding the original time sequence, where $r\in S_d$ specifies the span of the current combination. 
Consumption mapping is the process used to apply the First-order Relation to $\mathcal{C}^{(d)}_{p,r}$ in order to achieve $\mathcal{F}_{\mathcal{C}^{(d)}_{p,r}}$.
\end{definition}
In \textbf{Definition~\ref{def:con_mapping}}, only the non-zero values need to be considered since the zero values will lead to more combinations but will not assist the reconstruction.
$r\in S_d$ indicates the length of one combination in $\mathcal{C}^{(d)}_{p,r}$.
Formally, for $\forall d=1,2, \cdots, D$, we use $\boldsymbol{\mathscr{M}}^{(d)}$ to denote the consumption mapping results of $\mathbf{Y}^{(d)}$.
\begin{definition}[Shape Features]
For $\forall d=1, 2, \cdots, D$, given $\boldsymbol{\mathscr{M}}^{(d)}$, the shape features of device $d$ denoted by $\boldsymbol{\mathcal{S}}^{(d)}$ corresponds to the set of unique elements in $\boldsymbol{\mathscr{M}}^{(d)}$.
\end{definition}
Based on the observed Toilet consumption in Eq.~\eqref{eq:example_toilet}, by removing the redundant patterns in Eq.~\eqref{eq:example_toilet_approx}, the shape features for Toilet are as follows:
\begin{equation}\label{eq:example_toilet_shape_features} \equationsize
\left\{ 
	\left[
	\begin{array}{c}
	0  \\
	1\end{array}
	\right],
	\left[
	\begin{array}{c}
	1  \\
	0\end{array}
	\right],
	\left[
	\begin{array}{c}
	1  \\
	1\end{array}
	\right]
\right\}
\end{equation}
For a specific device , the complexity is significantly reduced from infinity ($\infty$) to $\sum_{q=1}^{|S_{d}|} (2^{S_{d,q}} - 1)$, where $S_{d,q}$ is the $q$th element in $S_d$. 
Taking Toilet as an example, without considering the locations of the intervals, the total number of combinations is $(2^1-1 + 2^2 - 1) = 4$.
\begin{algorithm}[H] \small
\caption{Shape Feature Discovery and Basis Initialization} \label{algorithm:SFDBC}
\begin{flushleft}
\textbf{Input:} $\mathbf{Y}^{(d)}$: consumption matrix for each device $d$; 
\\ 
\textbf{Output:} $\boldsymbol{\mathcal{S}}^{(d)}$: shape features for each device $d$; 
~~$\boldsymbol{\mathcal{H}}^{(d)}$: dictionaries for each device $d$
\end{flushleft}
\begin{algorithmic}[1]
\makeatletter\setcounter{ALG@line}{0}\makeatother
\For{$d \gets 1 \text{ to } D$}
	\State $\mathbf{Y}_+\gets\mathbf{Y}^{(d)} > 0$. \Comment{Extract the non-zero values while holding the original time order.}
	\State $\boldsymbol{\mathscr{M}}\gets$ Consumption mapping results of $\mathbf{Y}_+$.
	\State $\boldsymbol{\mathcal{S}}^{(d)}\gets$ Unique$\left(\boldsymbol{\mathscr{M}}\right)$.
	\State $\boldsymbol{\mathcal{H}}^{(d)}\gets$ Extend vectors in $\boldsymbol{\mathcal{S}}^{(d)}$ by filling all other intervals with zeroes.
	\State $\boldsymbol{\mathcal{H}}^{(d)}\gets \boldsymbol{\mathcal{H}}^{(d)} ./ \| \boldsymbol{\mathcal{H}}^{(d)}\|_2$. \Comment{Normalize all the basis functions.}
\EndFor 
\end{algorithmic}
\end{algorithm}
The process used to discover the devices' shape features is summarized in Algorithm~\ref{algorithm:SFDBC}.
The function ``Unique'' at Line 4 indicates the operation applied to remove all the redundant elements, and the operator ``$\mathbf{A} ./ b $'' in Line 6 represents dividing each element in $\mathbf{A}$ by the scalar variable $b$.

Generally, the initialized basis functions in Algorithm~\ref{algorithm:SFDBC} are not the bases that could lead to the most sparse coefficients, although they can provide near-zero construction errors.
For example, the optimal basis function (i.e., the one that achieves the most sparse coefficient) to reconstruct the consumption vector
{
\[ \left[ 0.8, 2.5 \right]^T \]
}
is the one with its exact normalization value
\begin{equation} \label{eq:normalized_basis}
\left[ 0.3048, 0.9524 \right]^T
\end{equation}	
This yields a single coefficient value of 2.6249.
Instead, using the bases generated by shape features, a perfect reconstruction is achieved by
\begin{equation} \label{eq:shape_basis}
0.8 * \left[1, 0
		\right]^T + 2.5 * \left[
					0, 1
					\right]^T
\end{equation}
producing the sum of coefficients $2.5+0.8=3.3 > 2.6249$.
On the other hand, the bases in Eq.~\eqref{eq:shape_basis} are easily generalized to reconstruct other two-interval time series, while the basis in Eq.~\eqref{eq:normalized_basis} lacks this capability.
We are thus motivated to balance these two capabilities by adding basis functions that can potentially yield the most sparse coefficients.
\\\\
\textbf{Basis Smoothness:}
The basis functions will be constructed to better adapt to variances in the consumption data.
The key is to extract the bases that are the normalization values of exact consumption.
If all the combinations of continuous consumption values are enumerated, a huge  number of bases will be generated, which would be prohibitive for fast learning.
\textbf{Cover} is provided to define the relationships of time series data to facilitate pruning the unnecessary candidates
\begin{definition}[Cover] \label{def:cover}
For any two bases $\mathbf{h}_i$ and $\mathbf{h}_j$, 
if $\mathbf{h}_k$ is the same as $\mathbf{h}_j$, then we say $\mathbf{h}_i$ is able to cover $\mathbf{h}_j$,
where $\mathbf{h}_k=\mathbf{h}_k^{\prime} ./ \|\mathbf{h}_k^{\prime}\|^2$, and for $n = 1, 2, \cdots, N$
\begin{equation}
h_{k,n}^{\prime} = 
\begin{cases}
		h_{i, n}  &, \; h_{j, n} > 0\\ 
		0  &, \; h_{j, n} == 0 
\end{cases}
\end{equation}
where $h_{i, n}, h_{j, n}, h^{\prime}_{k, n}$ are respectively the $n^{\text{th}}$ element in vectors $\mathbf{h}_i, \mathbf{h}_j, \mathbf{h}_k^{\prime}$.
\end{definition}
Those bases that can be covered by other bases are redundant since they are generated by excessive enumerations.
Given a time series consumption for Shower,
\begin{equation} \label{eq:ts_shower}
\left[
		17.28,
		9.61,
		1.69
\right]^T
\end{equation}
the anticipated basis is its normalized vector
\begin{equation} \label{eq:n_basis} 
\left[
		0.8708,
		0.4843,
		0.0852
	\right]^T
\end{equation}
In addition to the basis in Eq.~\eqref{eq:n_basis}, the time series in Eq.~\eqref{eq:ts_shower} might produce other bases due to redundant enumerations
\begin{equation}\label{eq:n_excessive_bases} \equationsize
\left[
	\begin{array}{c}
		0.8739 \\
		0.4860 \\
		0.0
	\end{array}
\right],
\left[
	\begin{array}{c}
		0.9953 \\
		0.0 \\
		0.0973
	\end{array}
\right], 
\left[
	\begin{array}{c}
		0.0 \\
		0.9849 \\
		0.1732
	\end{array}
\right] 
\end{equation}
The redundant bases in Eq.~\eqref{eq:n_excessive_bases} are already covered by the basis in Eq.~\eqref{eq:n_basis}.
The pruning process can thus be accelerated with the help of Cover's transitivity property:
\begin{proposition}[Transitivity] \label{prop:transitive}
For any three bases $\mathbf{h}_p$, $\mathbf{h}_q$ and $\mathbf{h}_r$, if $\mathbf{h}_p$ can cover $\mathbf{h}_q$ and $\mathbf{h}_q$ can cover $\mathbf{h}_r$, then $\mathbf{h}_p$ can cover $\mathbf{h}_r$.
\end{proposition}
\begin{proof}:
Let $\phi_1$ and $\phi_2$ respectively denote the set of non-zero positions in bases $\mathbf{h}_q$ and $\mathbf{h}_r$.
Since $\mathbf{h}_q$ can cover $\mathbf{h}_r$, $\phi_2 \subseteq \phi_1$.

Let $\mathbf{h}_{p, \phi_1}$ denote the vector: for $\forall n\in \phi_1$, the value is $h_{p, n}$, and the value is zero for all other positions.
Since $\mathbf{h}_p$ can cover $\mathbf{h}_q$, $\mathbf{h}_{p, \phi_1} ./ \|\mathbf{h}_{p, \phi_1}\|^2$ is the same as $\mathbf{h}_q$.
Similarly, $\mathbf{h}_{q, \phi_2}./\|\mathbf{h}_{q, \phi_2}\|^2$ is the same as $\mathbf{h}_r$.

$\phi_2 \subseteq \phi_1 \Longrightarrow \phi_2 \cap \phi_1=\phi_2$.
$\mathbf{h}_{p, \phi_2} ./ \|\mathbf{h}_{p, \phi_2}\|^2$ would be the same as $\mathbf{h}_{q, \phi_2}./\|\mathbf{h}_{q, \phi_2}\|^2$.
Thus $\mathbf{h}_{p}$ can cover $\mathbf{h}_{r}$.
\end{proof}
\begin{algorithm}[H] \small
\caption{Smooth Basis Functions} \label{algorithm:smooth_bases}
\begin{flushleft}
\textbf{Input:}~~$\mathbf{Y}^{(d)}$: consumption matrix for each device $d$; 
\\
\textbf{Output:} 
~~$\boldsymbol{\mathcal{H}}_s^{(d)}$: smoothed dictionaries for each device $d$
\end{flushleft}
\begin{algorithmic}[1]
\makeatletter\setcounter{ALG@line}{0}\makeatother
\For{$d \gets 1 \text{ to } D$}
	\State $\mathbf{Y}_+\gets\mathbf{Y}^{(d)} > 0$. \Comment{Extract the non-zero values while holding the original time order.}
	\State $\mathbf{C}\gets$ All possible combinations of elements in $\mathbf{Y}_+$ with the constraint of $S_d$ while holding the original time series order.
	\State $\mathbf{C}_h\gets$ Generate candidates of bases by normalizing the vectors in $\mathbf{C}$.
	\State $\boldsymbol{\mathcal{H}}_s^{(d)}\gets$ Remove all the bases in $\mathbf{C}_h$ that can be covered by others.
	\State $\boldsymbol{\mathcal{H}}_s^{(d)}\gets$ Extend vectors in $\boldsymbol{\mathcal{H}}_s^{(d)}$ by filling all other intervals with zeroes.
\EndFor 
\end{algorithmic}
\end{algorithm}
As shown in Algorithm~\ref{algorithm:smooth_bases}, the non-zero values are first extracted from the consumption matrix while holding the original time order.
Under the constraint of the span of devices, candidates are generated by enumerating all possible combinations,
after which the bases are finalized by normalizing the candidates and removing the bases covered by others.


\subsection{Bayesian Sparse Coding with Laplace Prior} \label{subsec:scg}
For each device, we have a Bayesian sparse coding model to capture the corresponding consumption patterns.
Without loss of generality, the labels for device and day can be removed.
Let $\mathbf{y}\in\mathbb{R}^{N\times 1}$ denote one day's water consumption of a particular device,
let $\mathbf{H}\in\mathbb{R}^{N\times M}$ denote the basis functions,
and let $\mathbf{u}$ denote 0-mean, and $\tau$-precision white noise.
The generative model is:
\begin{equation}\label{eq:gen}
\mathbf{y} = \mathbf{Hx} + \mathbf{u}
\end{equation}
The conditional probability of one interval's consumption is given by
\begin{equation} \label{eq:y_i}
P(y_i\mid \mathbf{x}, \tau, \mathbf{H}) = \mathcal{N}( y_i\mid \mathbf{H}_i\mathbf{x}, \tau^{-1} )
\end{equation}
Since $y_1, y_2, \cdots, y_N$ are i.i.d. variables, we have
\begin{equation} \label{eq:y}
P(\mathbf{y}\mid \mathbf{x}, \tau, \mathbf{H}) = \prod_{i=1}^N P(y_i\mid \mathbf{x}, \tau)
\end{equation}
In this model, to guarantee the sparseness, $x_j$ follows a Laplace distribution
\begin{equation}\label{eq:x_j}
\begin{split}
& P(x_j) = \frac{1}{2b}e^{ -\frac{|x_j|}{b} }
\end{split}
\end{equation}
Since $x_1, x_2, \cdots, x_M$ are i.i.d. variables, we have
\begin{equation}\label{eq:x}
\begin{split}
& P(\mathbf{x}) = \prod_{j=1}^{M} \frac{1}{2b}e^{ -\frac{|x_j|}{b} }
\end{split}
\end{equation}
The parametric form~\eqref{eq:x_j} provides a probabilistic generative description of the coefficient.
The prior distribution on $\tau$ follows a Gamma distribution with hyper-parameters $\alpha_0,\beta_0$:
\begin{equation}\label{eq:tau}
P(\tau\mid \alpha_0,\beta_0) = \text{Gamma}(\tau\mid\alpha_0,\beta_0)
\end{equation}
This model can be expressed as a directed graph, illustrated in Figure~\ref{fig:graphical_model}.
\begin{figure}[ht!]
\centering
\fbox{\includegraphics[width=0.4\textwidth]{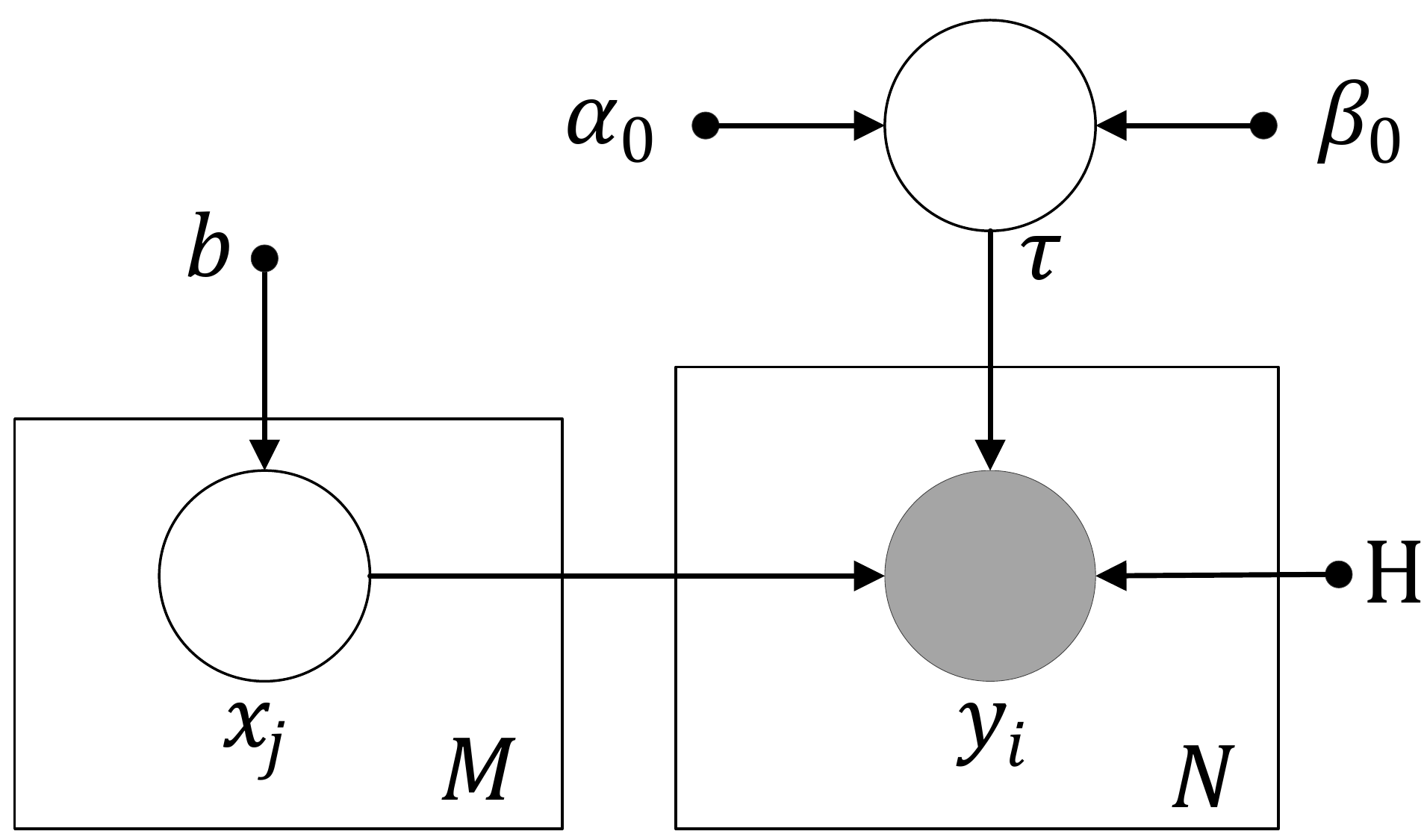}}
\caption{Representation of the generative model as
a directed acyclic graph. 
The observed variable $y_i$ is shown by the shaded node, 
while the latent variables $x_j$ and $\tau$ are shown by the circle.
The right box represents the $N$ independent consumption intervals from the data set,
while the left box represents the $M$ independent coefficients.
$\mathbf{b}$ and $\mathbf{H}$ are model parameters.
$\alpha_0$ and $\beta_0$ are hyperparameters.
}
\label{fig:graphical_model}
\end{figure}

\subsection{Learning the Discriminative Disaggregation Dictionaries} \label{subsec:bddm}
Section~\ref{subsec:scg} presents the Bayesian generative model for each device, and the basis functions (i.e., reconstruction dictionary) is optimized for reconstructing the individual consumption $\mathbf{y}^{(d)}$ where $d=1, 2, \cdots, D$, instead of optimized for disaggregating the aggregated consumption $\mathbf{\bar{y}}$.
Since the reconstruction dictionaries are not designed to enhance the disaggregation performance, the dictionaries may not be capable of effectively separating the aggregated data.
Based on the fact that distributions of coefficients are usually invariant over time or homes, we design a Bayesian discriminative model to train the disaggregation dictionaries using the aggregated data by holding the coefficients' parameters unchanged, targeting for enhancing the disaggregation performance.

Let $\bar{\mathbf{H}} = [ \mathbf{H}^{(1)}, \mathbf{H}^{(2)}, \cdots, \mathbf{H}^{(D)} ]$ denote the compound basis functions(i.e., disaggregation dictionaries),
where $\mathbf{H}^{(d)}, d=1, \cdots, D$ is the reconstruction dictionary trained using the Bayesian sparse coding model.
Let $\bar{\mathbf{y}} = \sum_{d=1}^D \mathbf{y}^{(d)}$ denote the aggregated consumption, and let \[ 
\bar{\mathbf{x}} = \left[ 
\begin{array}{c}
	\mathbf{x}^{(1)} \\
	\mathbf{x}^{(2)} \\
	\vdots  \\
	\mathbf{x}^{(D)}
\end{array} 
	\right]
 \]
denote the compound coefficients.
The key is to estimate the usages of individual devices: $\mathbf{\hat{y}}^{(1)}, \mathbf{\hat{y}}^{(2)}, \cdots, \mathbf{\hat{y}}^{(D)}$ from the total consumption $\mathbf{\bar{y}}$.

The aggregated consumption $\mathbf{\bar{y}}$ can be expressed as
\begin{equation}\label{eq:gen_bar_y}
\mathbf{\bar{y}} = \mathbf{\bar{H}} \mathbf{\bar{x}} + \mathbf{\bar{u}}
\end{equation}
where $\mathbf{\bar{u}}$ is the 0-mean, $\bar{\tau}$-precision white noise.
From the generative model for the aggregated consumption, we observe that $\mathbf{\bar{x}}$ denotes the overall coefficients and $\mathbf{\bar{H}}$ denotes the basis functions for constructing the aggregated consumption. 
The key here is to learn the discriminative $\mathbf{\bar{H}}$ for estimating individual devices' consumption.
Given the normalized basis functions, the active coefficients mainly depend on the consumption amplitude of individual devices, while the non-active coefficients are near-zero values.
The parameters of coefficients learned with individual devices' data should be optimal if they are to adequately represent the invariant patterns captured by distributions of coefficients since they are learned with device level data.
The discriminative capability of $\mathbf{\bar{H}}$ can thus be extended through training with the aggregated data while keeping parameters of coefficients unchanged (i.e., $b^{(1)}, b^{(2)}, \cdots, b^{(D)}$).

The conditional probability of one interval's aggregated consumption is given by $\bar{y}_i$
\begin{equation}\label{eq:bar_y|bar_x,bar_tau}
P(\bar{y}_i\mid \mathbf{\bar{x}, \mathbf{\bar{H}}}, \bar{\tau}) = \mathcal{N}( \bar{y}_i\mid \mathbf{\bar{H}}_i\mathbf{\bar{x}}, \bar{\tau}^{-1} )
\end{equation}
The aggregated interval based consumption $\bar{y}_1, \bar{y}_2, \cdots, \bar{y}_N$ are i.i.d. variables:
\begin{equation}\label{eq:bar_ay|bar_x,bar_tau}
P(\mathbf{\bar{y}}\mid \mathbf{\bar{x}, \mathbf{\bar{H}}}, \bar{\tau}) =
\prod_{i=1}^N P(\bar{y}_i\mid \mathbf{\bar{x}, \mathbf{\bar{H}}}, \bar{\tau})
\end{equation}
The prior distribution on $\bar{\tau}$ is given by a Gamma distribution:
\begin{equation}\label{eq:bar_tau}
P(\bar{\tau}\mid\bar{\alpha}_0, \bar{\beta}_0) = \text{Gamma}(\bar{\tau}\mid \bar{\alpha}_0, \bar{\beta}_0)
\end{equation}

\section{Inference and Learning} \label{sec:gs_infer}
Section~\ref{subsec:i_bscm} introduces the inference process for the Bayesian sparse coding model, and Section~\ref{subsec:i_ddm} presents the learning over the discriminative disaggregation model.
The predictive density is evaluated in Section~\ref{subsec:epd}.

\subsection{Inference on Bayesian Sparse Coding Model}\label{subsec:i_bscm}
EM algorithm~\cite{Dempster:77} is applied to maximize the likelihood function and estimate the model parameters.
The evaluation function is:
\begin{equation}\label{eq:em_evaluate}
Q(\boldsymbol{\theta}, \boldsymbol{\theta}^{\text{old}}) = \int_{\mathbf{W}} P(\mathbf{W}\mid \mathbf{y}, \boldsymbol{\theta}^{\text{old}}) \ln P(\mathbf{y}, \mathbf{W}\mid \boldsymbol{\theta})
\end{equation}
where $\boldsymbol{\theta}=\{\mathbf{b}, \mathbf{H}\}$ denote model parameters,
$\mathbf{W}=\{\mathbf{x}, \tau\}$ denotes the latent variables.
In the E step, $Q(\boldsymbol{\theta}, \boldsymbol{\theta}^{\text{old}})$ is evaluated, while in the M step, $Q(\boldsymbol{\theta}, \boldsymbol{\theta}^{\text{old}})$ is maximized with respect to $\boldsymbol{\theta}$.
The key problem is to evaluate the expectation of the joint probability over the posterior distribution of latent variables.
A Gibbs sampling~\cite{Geman:84} based inference method is designed to estimate the expectation value:
\begin{enumerate}
	\item $t=0$. Set initial values $\mathbf{x}^{(0)}, \tau^{(0)}$
	
	\item Generate $\tau^{(t+1)}\sim P(\tau\mid \mathbf{y}, \mathbf{x}^{(t)}, \boldsymbol{\theta}^{\text{old}})$

	\item Generate $\mathbf{x}^{(t+1)}\sim P(\mathbf{x}\mid \mathbf{y}, \tau^{(t+1)}, \boldsymbol{\theta}^{\text{old}})$
	
	\item $t = t + 1$. Go to 2 if $t <= T$.
\end{enumerate}
where $T$ is the total number of samples that need to be generated.
Then, we have
\begin{equation}\label{eq:e_evaluation}
Q(\boldsymbol{\theta}, \boldsymbol{\theta}^{\text{old}}) \simeq \frac{1}{T-s}\sum_{t=s+1}^T \ln P(\mathbf{y}, \mathbf{W}^{(t)}\mid \boldsymbol{\theta})
\end{equation}
where $s$ is the threshold for abandoning the starting samples to remove the effect of bad initializations.

Now we need to derive the distribution forms for all hidden variables.
The distribution of $\tau$ given observations and other hidden variables is
\begin{equation}\label{eq:tau|others_lap}
\begin{split}
P(\tau\mid \mathbf{y, x}, \boldsymbol{\theta}) 
& = \frac{ P(\mathbf{y, x},\tau, \boldsymbol{\theta}) }{P(\mathbf{y, x}\mid \boldsymbol{\theta})}
\propto \text{Gamma}(\tau\mid \alpha_N, \beta_N)
\end{split}
\end{equation}
where $\alpha_N = \alpha_0 + \frac{N}{2}, \beta_N=\beta_0+\frac{1}{2}\sum_{i=1}^N (y_i-\mathbf{H}_i\mathbf{x})^2$.
The log form of the distribution of $x_j$ given observations and other hidden variables is
\begin{equation}\label{eq:x_j|others2_lap}
\begin{split}
&\ln P(x_j\mid \mathbf{y}, \mathbf{x}\setminus x_j, \tau, \boldsymbol{\theta}) 
\propto \sum_{i=1}^N 
-\frac{\tau}{2}\bigg[ \big(H_{ij}x_j\big)^2 - 2H_{ij}x_j\big( y_i-\sum_{j^{\prime}\neq j}H_{ij^{\prime}}x_{j^{\prime}} \big) \bigg]
 - \frac{|x_j|}{b}  \bigg]
\end{split}
\end{equation}

In the maximization step (\textbf{M} step), we intend to maximize $Q(\boldsymbol{\theta}, \boldsymbol{\theta}^{\text{old}})$ with respect to $\boldsymbol{\theta}$.
From Eq.~\eqref{eq:e_evaluation}, we know,
\begin{equation} \label{eq:derivation_lap}
\begin{split}
Q(\boldsymbol{\theta}, \boldsymbol{\theta}^{\text{old}}) 
& \simeq \frac{1}{T-s}\sum_{t=s+1}^T \ln P(\mathbf{y}, \mathbf{W}^{(t)}\mid \boldsymbol{\theta}) = \frac{1}{T-s}\sum_{t=s+1}^T \left( \mathcal{F}_1 + \mathcal{F}_2
+ \mathcal{F}_3 \right)
\end{split}
\end{equation}
where,
\begin{equation}\label{eq:notation_lap}
\begin{split}
&\mathcal{F}_1 = \sum_{i=1}^N \bigg[ 0.5 \ln \frac{\tau^{(t)}}{2\pi} - \frac{\tau^{(t)}}{2}\bigg( y_i^2 - 2y_1\sum_{j=1}^M H_{ij}x_j^{(t)} + (\sum_{j=1}^M H_{ij} x_j^{(t)})^2 \bigg)  \bigg]
\\
& \mathcal{F}_2 = \sum_{j=1}^M\left( \ln{\frac{1}{2b}} - \frac{|x_j^{(t)}|}{b} \right)
\\
&\mathcal{F}_3 = \alpha_0\ln \beta_0 - \ln \Gamma(\alpha_0) + (\alpha_0-1)\ln\tau^{(t)}-\beta_0\tau^{(t)}
\end{split}
\end{equation}
The derivatives of these three functions on $H_{ij}$ and $b$ are performed to complete the maximization step:
\begin{equation}\label{eq:maximization_H_ij_lap}
\begin{split}
&\frac{\partial \mathcal{F}_1}{\partial H_{ij}} = -\frac{\tau^{(t)}}{2} \bigg[ 
	2 (x_j^{(t)})^2 H_{ij}  - 2y_ix_j^{(t)} + 2\sum_{j^{\prime}\neq j} H_{ij^{\prime}} x_{j^{\prime}}^{(t)} x_j^{(t)}
\bigg]
\end{split}
\end{equation}
\begin{equation}\label{eq:maximization_b_lap}
\begin{split}
&\frac{\partial \mathcal{F}_2}{\partial b} = \sum_{j=1}^M \left( - \frac{1}{b} + \frac{|x_j^{(t)}|}{b^2} \right)
\end{split}
\end{equation}
Hence, the learning rules can be achieved from above equations,
\begin{equation} \label{eq:learning_rules_lap}
\begin{split}
& H_{ij} = \frac{1}{T-s}\sum_{t=s+1}^T\frac{ y_i - \sum_{j^{\prime}\neq j} H_{ij^{\prime}} x_{j^{\prime}}^{(t)} }{(x_j^{(t)})^2}
\\
& b = \frac{1}{T-s}\sum_{t=s+1}^T \frac{\sum_{j}^{m}|x_j^{(t)}|}{M}
\end{split}
\end{equation}

\subsubsection*{Hyperparameters Optimization} \label{subsec:hyper_optimization}
In addition to optimizing the model parameters, the hyperparameters can be updated to further enlarge the likelihood.
Following the maximization step, $\mathcal{F}_3$ can be maximized with respect to $\alpha_0, \beta_0$.
The derivatives are given by,
\begin{equation}\label{eq:m_alpha_0_lap}
\frac{\partial \mathcal{F}_3}{\partial \alpha_0} = \ln \beta_0 - \psi(\alpha_0) + \ln\tau^{(t)}
\end{equation}
\begin{equation}\label{eq:m_beta_0_lap}
\frac{\partial \mathcal{F}_3}{\partial \beta_0} =  \frac{\alpha_0}{\beta_0} -\tau^{(t)}
\end{equation}
Then, the learning rules are,
\begin{equation} \label{eq:hyper_learning_rules_lap}
\begin{split}
& \psi(\alpha_0) =  \ln \beta_0 + \frac{1}{T-s}\sum_{t=s+1}^T \ln\tau^{(t)}
\\
& \beta_0 = \frac{1}{T-s}\sum_{t=s+1}^T \frac{\alpha_0}{ \tau^{(t)} }
\end{split}
\end{equation}

As shown in Algorithm~\ref{algorithm:parameters_learning}, the bases generated in Algorithms~\ref{algorithm:SFDBC} and~\ref{algorithm:smooth_bases} are combined together to form the initialized basis functions for learning.
Note that good initializations of the basis functions could significantly reduce the number of iterations and improve the disaggregation performance. 
\begin{algorithm}[H] \small
\caption{Learning Parameters for Individual Devices} 
\label{algorithm:parameters_learning}
\begin{flushleft}
\textbf{Input:} $\mathbf{Y}^{(d)}$: consumption matrix for each device $d$.
\\
\textbf{Output:} $\mathbf{H}^{(d)}$: basis functions for each device $d$; 
~~$b^{(d)}, \alpha_0^{(d)}, \beta_0^{(d)}$: parameters for each device $d$.
\end{flushleft}
\begin{algorithmic}[1]
\makeatletter\setcounter{ALG@line}{0}\makeatother
\For{$d \gets 1 \text{ to } D$}
	\State $\mathbf{H}^{(d)}\gets \boldsymbol{\mathcal{H}}^{(d)}\cup \boldsymbol{\mathcal{H}}_s^{(d)}$. \Comment{Initialize the basis functions with the bases generated by Algorithms~\ref{algorithm:SFDBC} and~\ref{algorithm:smooth_bases}.}
	\Repeat
		\State Generate the samples for $\tau, \mathbf{x}$ with schemas defined in Section~\ref{subsec:i_bscm}.
		\State Update $\mathbf{H}^{(d)}, b^{(d)}$ using Eq.~\eqref{eq:learning_rules_lap}.
		\State Update $\alpha_0^{(d)}, \beta_0^{(d)}$ using Eq.~\eqref{eq:hyper_learning_rules_lap}.
		\State Evaluate $Q(\boldsymbol{\theta}, \boldsymbol{\theta}^{(\text{old})})$ using Eq.~\eqref{eq:derivation_lap}. 
	\Until{Convergence}
\EndFor 
\end{algorithmic}
\end{algorithm}

\subsection{Inference on the Discriminative Disaggregation Model}\label{subsec:i_ddm}
Now we are ready to learn the discriminative disaggregation model to enhance the discriminative power of the basis function $\mathbf{\bar{H}}$.
Similarly, EM algorithm~\cite{Dempster:77} is applied and the evaluation function is given by
\begin{equation}\label{eq:eva_dEM_lap}
Q(\mathbf{\bar{H}}, \mathbf{\bar{H}}^{\text{old}}) = \int_{\mathbf{\bar{W}}} P(\mathbf{\bar{W}}\mid \mathbf{\bar{y}}, \mathbf{\bar{H}}^{\text{old}}) \ln P(\mathbf{\bar{y}}, \mathbf{\bar{W}}\mid \mathbf{\bar{H}})
\end{equation}
where $\mathbf{\bar{W}}=\{\mathbf{\bar{x}}, \bar{\tau}\}$ denote the latent variables.
In the E step, $Q(\mathbf{\bar{H}}, \mathbf{\bar{H}}^{\text{old}})$ is evaluated, while in the M step, $Q(\mathbf{\bar{H}}, \mathbf{\bar{H}}^{\text{old}})$ is maximized with respect to $\mathbf{\bar{H}}$.

Gibbs sampling~\cite{Geman:84} is applied to estimate the expectation value:
\begin{enumerate}
	\item $t=0$. Set initial values $\mathbf{\bar{x}}^{(0)}, \bar{\tau}^{(0)}$
	
	\item Generate $\bar{\tau}^{(t+1)}\sim P(\bar{\tau}\mid \mathbf{\bar{y}}, \mathbf{\bar{x}}^{(t)}, \mathbf{\bar{H}}^{\text{old}})$
	
	\item Generate $\mathbf{\bar{x}}^{(t+1)}\sim P(\mathbf{\bar{x}}\mid \mathbf{\bar{y}}, \bar{\tau}^{(t+1)}, \mathbf{\bar{H}}^{\text{old}})$
	
	\item $t = t + 1$. Go to 2 if $t <= T$.
\end{enumerate}
where $T$ is the total number of samples that need to be generated.
Then, we have
\begin{equation}\label{eq:e_deva_lap}
Q(\mathbf{\bar{H}}, \mathbf{\bar{H}}^{\text{old}}) \simeq \frac{1}{T-s}\sum_{t=s+1}^T \ln P(\mathbf{\bar{y}}, \mathbf{\bar{W}}^{(t)}\mid \mathbf{\bar{H}}^{\text{old}})
\end{equation}
where $s$ is the threshold for abandoning the starting samples to remove the effect of bad initializations.

The distribution of $\bar{\tau}$ given the aggregated consumption and other latent variables is
\begin{equation}\label{eq:bar_tau|others_lap}
\begin{split}
P(\bar{\tau}\mid \mathbf{\bar{y}, \bar{x}}, \mathbf{\bar{H}}) 
& = \frac{ P(\mathbf{\bar{y}, \bar{x}}, \bar{\tau} \mid \mathbf{\bar{H}}) }{P(\mathbf{\bar{y}, \bar{x}}\mid \mathbf{\bar{H}})}
\propto \text{Gamma}(\bar{\tau}\mid \bar{\alpha}_N, \bar{\beta}_N)
\end{split}
\end{equation}
where $\bar{\alpha}_N = \bar{\alpha}_0 + \frac{N}{2}, \bar{\beta}_N=\bar{\beta}_0+\frac{1}{2}\sum_{i=1}^N (\bar{y}_i-\mathbf{\bar{H}}_i\mathbf{\bar{x}})^2$.
The distribution of $\bar{\mathbf{x}}$ given the aggregated consumption and other latent variables is
\begin{equation}\label{eq:bar_x|others_lap}
\begin{split}
P(\mathbf{\bar{x}}\mid \mathbf{\bar{y}}, \bar{\tau}, \mathbf{\bar{H}}) 
& = \frac{ P(\mathbf{\bar{y}}, \bar{\tau}, \mathbf{\bar{x}} \mid \mathbf{\bar{H}}) }{P(\mathbf{\bar{y}}, \bar{\tau}, \mid \mathbf{\bar{H}})}
\propto P(\mathbf{\bar{y}} \mid \mathbf{\bar{x}}, \bar{\tau}, \mathbf{\bar{H}}) P(\mathbf{\bar{x}})
\end{split}
\end{equation}

During the course of maximization, $Q(\mathbf{\bar{H}}, \mathbf{\bar{H}}^{\text{old}})$ is maximized with respect to $\mathbf{\bar{H}}$ while holding the coefficients' parameters ($b^{d}, d=1,2,\cdots, D$) unchanged.
From Eq.~\eqref{eq:e_deva_lap}, we get
\begin{equation} \label{eq:d_derivation_lap}
\begin{split}
Q(\mathbf{\bar{H}}, \mathbf{\bar{H}}^{\text{old}}) 
& \simeq \frac{1}{T-s}\sum_{t=s+1}^T \ln P(\mathbf{\bar{y}}, \mathbf{W}^{(t)}\mid \mathbf{\bar{H}}) 
= \frac{1}{T-s}\sum_{t=s+1}^T \left(\bar{\mathcal{F}_1} + \bar{\mathcal{F}_2} + \bar{\mathcal{F}_3}\right)
\end{split}
\end{equation}
where $\bar{b} = b^{(\bar{d})}, \bar{d} = \lceil \frac{j}{M} \rceil$, and $\bar{M} = M\times D$ is the total number of elements in $\mathbf{\bar{x}}$, and
\begin{equation}\label{eq:d_notation_lap}
\begin{split}
&\mathcal{\bar{F}}_1 = \sum_{i=1}^N \bigg[ 0.5 \ln \frac{\bar{\tau}^{(t)}}{2\pi} - \frac{\bar{\tau}^{(t)}}{2}\bigg( \bar{y}_i^2 - 2\bar{y}_i\sum_{j=1}^{\bar{M}} \bar{H}_{ij}\bar{x}_j^{(t)} + (\sum_{j=1}^{\bar{M}} \bar{H}_{ij} \bar{x}_j^{(t)})^2 \bigg)  \bigg] 
\\
& \mathcal{\bar{F}}_2 = \sum_{j=1}^{\bar{M}} \left( \ln{\frac{1}{2\bar{b}}} - \frac{|\bar{x}_j^{(t)}|}{\bar{b}} \right)
\\
&\mathcal{\bar{F}}_3 = \bar{\alpha}_0\ln \bar{\beta}_0 - \ln \Gamma(\bar{\alpha}_0) + (\bar{\alpha}_0-1)\ln\bar{\tau}^{(t)}-\bar{\beta}_0\bar{\tau}^{(t)}
\end{split}
\end{equation}
Now we derive the derivatives of these three functions on $\bar{H}_{ij}$:
\begin{equation}\label{eq:d_maximization_H_ij_lap}
\begin{split}
&\frac{\partial \mathcal{\bar{F}}_1}{\partial \bar{H}_{ij}} = -\frac{\bar{\tau}^{(t)}}{2} \bigg[ 
	2 (\bar{x}_j^{(t)})^2 \bar{H}_{ij}  - 2\bar{y}_i\bar{x}_j^{(t)} + \sum_{j^{\prime}\neq j} \bar{H}_{ij^{\prime}} \bar{x}_{j^{\prime}}^{(t)} \bar{x}_j^{(t)}
\bigg]
\end{split}
\end{equation}
The learning rules can be achieved from above equation:
\begin{equation} \label{eq:d_learning_rules_lap}
\begin{split}
&\bar{H}_{ij} = \frac{1}{T-s}\sum_{t=s+1}^T\frac{ 2\bar{y}_i - \sum_{j^{\prime}\neq j} \bar{H}_{ij^{\prime}} \bar{x}_{j^{\prime}}^{(t)} }{2\bar{x}_j^{(t)}}
\end{split}
\end{equation}

\textbf{Optimizing $\bar{\alpha}_0$ and $\bar{\beta}_0$:} \label{subsec:d_hyper_optimization_lap}
Following the maximization step, $\mathcal{\bar{F}}_3$ can be maximized with respect to $\bar{\alpha}_0, \bar{\beta}_0$.
The derivatives of $\mathcal{\bar{F}}_3$ over these two parameters are
\begin{equation}\label{eq:d_m_alpha_0_lap}
\frac{\partial \mathcal{\bar{F}}_3}{\partial \bar{\alpha}_0} = \ln \bar{\beta}_0 - \psi(\bar{\alpha}_0) + \ln\bar{\tau}^{(t)}
\end{equation}
\begin{equation}\label{eq:d_m_beta_0_lap}
\frac{\partial \mathcal{\bar{F}}_3}{\partial \bar{\beta}_0} =  \frac{\bar{\alpha}_0}{\bar{\beta}_0} -\bar{\tau}^{(t)}
\end{equation}
The learning rules are
\begin{equation} \label{eq:d_hyper_learning_rules_lap}
\begin{split}
& \psi(\bar{\alpha}_0) =  \ln \bar{\beta}_0 + \frac{1}{T-s}\sum_{t=s+1}^T \ln\bar{\tau}^{(t)}
\\
& \bar{\beta}_0 = \frac{1}{T-s}\sum_{t=s+1}^T \frac{\bar{\alpha}_0}{ \bar{\tau}^{(t)} }
\end{split}
\end{equation}

As shown in Algorithm~\ref{algorithm:discriminative_training}, we first initialize the discriminative basis functions by combining the bases generated in Algorithms~\ref{algorithm:parameters_learning}.
With an iterative process, the discriminative capability of the disaggregation dictionaries are enhanced.
\begin{algorithm}[H] \small
\caption{Learning the Discriminative Parameters for Disaggregation} 
\label{algorithm:discriminative_training}
\begin{flushleft}
\textbf{Input:} $\mathbf{\bar{Y}}$: aggregated consumption matrix; 
\\
\textbf{Output:} $\mathbf{\bar{H}}$: the discriminative basis functions;
$\bar{\alpha}_0, \bar{\beta}_0$: Gamma distribution's parameters.
\end{flushleft}
\begin{algorithmic}[1]
\makeatletter\setcounter{ALG@line}{0}\makeatother
\State $\mathbf{\bar{H}}\gets [\mathbf{H}^{(1)}, \cdots, \mathbf{H}^{(D)}]$. \Comment{Initialize the basis functions with learned bases in Algorithm~\ref{algorithm:parameters_learning}.}
\Repeat
	\State Generate the samples for $\bar{\tau}, \bar{\mathbf{x}}$ with schemas defined in Section~\ref{subsec:i_ddm}.
	\State Update $\mathbf{\bar{H}}$ using Eq.~\eqref{eq:d_learning_rules_lap}.
	\State Update $\bar{\alpha}_0, \bar{\beta}_0$ using Eq.~\eqref{eq:d_hyper_learning_rules_lap}.
	\State Evaluate $Q(\mathbf{H}, \mathbf{H}^{(\text{old})})$ using Eq.~\eqref{eq:d_derivation_lap}. 
\Until{Convergence}
\end{algorithmic}
\end{algorithm}

\subsection{The Evaluation of the Predictive Density} \label{subsec:epd}
Now we intend to estimate the values of $\mathbf{\bar{X}}$.
\begin{equation}\label{eq:eva_PD_lap}
\begin{split}
P\left( \hat{\mathbf{y}}^{(1)}, \hat{\mathbf{y}}^{(2)}, \cdots, \hat{\mathbf{y}}^{(D)} \mid \mathbf{\bar{y}}, \boldsymbol{\bar{\theta}} \right)
& = \int P\left( \hat{\mathbf{y}}^{(1)}, \hat{\mathbf{y}}^{(2)}, \cdots, \hat{\mathbf{y}}^{(D)} \mid \mathbf{\bar{W}}, \boldsymbol{\bar{\theta}} \right) 
P\left( \mathbf{\bar{W}} \mid \mathbf{\bar{y}}, \boldsymbol{\bar{\theta}} \right) d\mathbf{\bar{W}}
\\
& = \mathbb{E}_{\mathbf{\bar{W}}\mid \mathbf{\bar{y}}} \left[ \prod_{d=1}^D P( \hat{\mathbf{y}}^{(d)} \mid \mathbf{\bar{W}}, \boldsymbol{\bar{\theta}} ) \right]
\end{split}
\end{equation}
Sampling method can be used to estimate the expectation of $\prod_{d=1}^D P( \hat{\mathbf{y}}^{(d)} \mid \mathbf{\bar{W}}, \boldsymbol{\bar{\theta}} )$ over the posterior distribution.
Suppose we get a series of samples $\mathbf{\bar{W}}^{(s+1)}, \mathbf{\bar{W}}^{(s+2)}, \cdots, \mathbf{\bar{W}}^{(T)}$ (the first $s$ number of samples have been discarded to remove the effect of bad initialization), then the predictive density is
\begin{equation}\label{eq:eva_PD_sampling_lap}
\begin{split}
P\left( \hat{\mathbf{y}}^{(1)}, \hat{\mathbf{y}}^{(2)}, \cdots, \hat{\mathbf{y}}^{(D)} \mid \mathbf{\bar{y}} \right)
& \simeq \prod_{d=1}^D \left[ \frac{1}{T-s} \sum_{t=s+1}^T  P( \hat{\mathbf{y}}^{(d)} \mid \mathbf{\bar{W}}^{(t)}, \boldsymbol{\bar{\theta}} ) \right]
\\
& = \prod_{d=1}^D \left[ \frac{1}{T-s} \sum_{t=s+1}^T \prod_{i=1}^N  
\mathcal{N}(\hat{y}_i^{(d)} \mid \mathbf{\bar{H}}^{(d)} \bar{\mathbf{x}}^{(d,t)}, \bar{\tau}^{(t)} )\right]
\\
\end{split}
\end{equation}
And the mode of $\hat{y}_i^{(d)}$ is 
\begin{equation}\label{eq:y_mode_lap}
\hat{y}_i^{(d)} = \frac{1}{T-s}\sum_{t=s+1}^T \mathbf{\bar{H}}^{(d)} \bar{\mathbf{x}}^{(d,t)}
\end{equation}

\section{Experimental Evaluations} \label{sec:experiment}
Comprehensive experiments on the proposed models were conducted in order to evaluate the disaggregation performance.
Section~\ref{subsec:dataset} introduces the experimental design and setup.
Section~\ref{subsec:scalable-syn} evaluates model performance using synthetic datasets of various sizes.
Section~\ref{subsec:large-real} evaluates model performance using a large scale real world dataset.

\subsection{Dataset and Setup} \label{subsec:dataset}
\textbf{Dataset:} 
A real-world dataset was collected by Aquacraft~\cite{Mayer:Dataset}, consisting of 1,959,817
water use events recorded during a two-year study from 1,188 households across 12 study sites, including Boulder,  Denver, etc.
Each device was labeled with one of 17 categories, and 5 common device types were considered in the experiments: Faucet, Dishwasher, Toilet, Shower, and Clothes Washer.
Since the widely deployed smart meters report at a low sample rate~\cite{Chen11}, the event records were generalized into time series with a sample rate of 1/900 Hz.
\\\\
\textbf{Baselines: } 
The two proposed models, BDSC-LP (Bayesian Discriminative Sparse Coding with Laplace Prior) and BDSC-LP+SF (Bayesian Discriminative Sparse Coding with Laplace Prior and Shape Features), were compared with the following baselines.
The BDSC-LP model is the method without using the shape features for the initialization of the basis functions.
The first baseline is Discriminative Disaggregation Sparse Coding (DDSC)~\cite{Kolter10} with Shape Features (SF), i.e., DDSC+SF, which uses shape features for the learning of DDSC's basis functions.
An approach combining DDSC with its extensions Total Consumption Priors (TCP) and Group Lasso (GL), i.e., DDSC+TCP+GL is the third baseline.
The third baseline is DDSC model, and 
the final baseline used for comparison is the Factorial Hidden Markov Model (FHMM)~\cite{Kolter:112,Kim:fhmm}.
\\\\
\textbf{Evaluation Metrics: }
Both whole-home and device level evaluation metrics are inspected, and the whole-home level disaggregation capability measured utilizing Accuracy~\cite{Kolter10} and Normalized Disaggregation Error (NDE)~\cite{Kolter:12}: Accuracy evaluates the total-day accuracy of the estimation methods, while NDE measures how well the models separate individual devices' consumption from the aggregated consumption
\begin{equation} \label{eq:accuracy}
\text{Accuracy} = \frac{\sum_{d, p} \min{  \left( \left\Vert \mathbf{y}^{(d)}_{\cdot, p} \right\Vert_1, \left\Vert \hat{\mathbf{y}}^{(d)}_{\cdot, p} \right\Vert_1\right) } }{\sum_{i, p} (\bar{y})_{i,p}  }
\end{equation}
\begin{equation} \label{eq:nde}
\text{NDE} = \sqrt{ \sum_{d, p} \left( \frac{ \left\Vert \mathbf{y}^{(d)}_{\cdot, p} - \hat{\mathbf{y}}^{(d)}_{\cdot, p} \right\Vert_2^2 }{  \left\Vert \mathbf{y}^{(d)}_{\cdot, p} \right\Vert_2^2 } \right)}
\end{equation}
where $\hat{\mathbf{y}}^{(d)}_{\cdot, p}$ is the estimated consumption for device $d$ at the $p^{\text{th}}$ day.

With respect to device level evaluation, the quantitative precision, recall and f-measure are defined: the precision is the fraction of disaggregated consumption that is correctly separated, recall is the fraction of true device level consumption that is successfully separated, and the F-measure for device $d$ is: $F(d) = 2\times\frac{ \text{Precision}(d)\times \text{Recall}(d) }{\text{Precision}(d) + \text{Recall}(d)}$, where
\begin{equation}\label{eq:precision}
\text{Precision}(d) = \frac{ \sum_{i, p} \min{ \left( y^{(d)}_{i, p}, \hat{y}^{(d)}_{i, p} \right)   } }{ \sum_{p, i} \hat{y}^{(d)}_{i, p}   }
\end{equation}
\begin{equation}\label{eq:recall}
\text{Recall}(d) = \frac{ \sum_{i, p} \min{ \left( y^{(d)}_{i, p}, \hat{y}^{(d)}_{i, p} \right)   } }{ \sum_{p, i} y^{(d)}_{i, p}   }
\end{equation}
where $\hat{y}^{(d)}_{i, p}$ is the estimated consumption for device $d$ at the $i^{\text{th}}$ interval in the $p^{\text{th}}$ day.
Additionally, the average F-measure is used to evaluate the models' overall disaggregation performance: $AF = \frac{1}{D} \sum_{d=1}^D F(d)$
\\\\
\textbf{The Generation of a Representative Synthetic Dataset:}
The goal here is to design a data generator to produce a representative synthetic dataset with a sampling rate of 1/900 Hz. 
The generator consists of three components: event dictionary construction, frequency pattern learning and data generation.
\\
\indent \textit{(1). \textbf{Event dictionary construction}: } 
Five water events (corresponding to five water devices) are considered: Faucet, Dishwasher, Toilet, Shower and Clothes Washer.
The event dictionary was built based on the real world dataset, where each event type pointed to a set of event records for this particular event type. 
\\
\indent \textit{(2). \textbf{Learning frequency patterns} : }
The daily and interval frequency of events were statistically calculated, where daily frequency was the number of events that happen in one day while interval frequency was the number of events starting from the corresponding interval.
Based on the assumption that daily frequency of events followed a Poisson distribution, the data was used to fit a Poisson distribution with a maximum-likelihood estimation, and the estimated parameters of all events are shown in Table~\ref{table:mle_poisson}.
The Cumulative Distribution Function (CDF) of the interval frequency was estimated with a kernel density function
, and the CDFs of events are illustrated in Figure~\ref{fig:cdf_devices} (located at~\ref{subsec:appC1}).
\\
\indent \textit{(3). \textbf{Data generation}: }
The data was generated day by day.
For each day, the Poisson distribution, which was trained with daily frequency, was first used to sample the number of events for one day.
Then with the CDFs learned with the starting interval frequency, the starting intervals of the events were sampled for a particular day.
Finally, the event records were randomly selected from the event dictionary.
Using this procedure, the simulation data was generated for 50, 100, 400, 800 and 1000 days to perform scalable evaluations.

\begin{table}[ht!]\scriptsize 
\centering
\begin{tabular}{c|c|c|c|c|c}
\hline
\diagbox[width=3.0cm]{\raisebox{-2pt}{\hspace*{0.7cm}{\scriptsize Parameter}}}{\raisebox{-4pt}{\hspace*{2.5cm}{\scriptsize Event Type}}} & 
Faucet & Dishwasher & Toilet & Shower & Clothes Washer  \\
\hline
{\scriptsize $\lambda$}  & 42.0856 & 1.0784 & 12.9203 & 2.3668 & 2.1761  \\
\hline
\end{tabular}
\caption{
Statistically estimated parameters for Poisson distribution, where $\lambda$ was the expected value of daily frequency.}
\label{table:mle_poisson}
\end{table}

\subsection{Performance Evaluations over Synthetic Data}\label{subsec:scalable-syn}
Various sizes of synthetic data were used to evaluate the proposed models (BDSC-LP+SF and BDSC-LP), reporting both the whole-home level performance and their comparisons with the baselines.
All the methods were evaluated using the synthetic datasets varying from 50 to 1000 days.
For each data size, 10-fold cross validation is applied, and the mean$\pm$std of Avg. F-measure, Accuracy and NDE are shown in Table~\ref{table:syn_scal}.
The overall performance of our proposed models was better than that of the baselines for every data size.
\begin{table}[htb!] \footnotesize 
\centering
\begin{tabular}{c|ccccc}
\hline
\multirow{1}* {
\diagbox[width=2.2cm]{\raisebox{-0.5pt}{\hspace*{0.0cm}{ Methods}}}{\raisebox{-1pt}{\hspace*{-1.0cm}{ Days}}}}  
& \multirow{2}* {50} & 
  \multirow{2}* {100} & 
  \multirow{2}* {400} & 
  \multirow{2}* {800} & 
  \multirow{2}* {1000} \\
& & & & & \\
\hline

\multirow{3}*{{\textbf{BDSC-LP+SF}}}
& \textbf{0.3092}$\pm$0.0859 & \textbf{0.3479}$\pm$0.0643 & \textbf{0.5064}$\pm$0.0960 & \textbf{0.5297}$\pm$0.0826 & \textbf{0.5307}$\pm$0.0422 \\
& \textbf{0.6195}$\pm$0.0692 & 0.6624$\pm$0.0376 & \textbf{0.7594}$\pm$0.0674 & 0.7465$\pm$0.0614 & \textbf{0.7943}$\pm$0.0521 \\
& \textbf{0.8584}$\pm$0.0914 & \textbf{0.8060}$\pm$0.0744 & \textbf{0.7022}$\pm$0.0872 & 0.7133$\pm$0.0725 & 0.6932$\pm$0.0421 \\

\hline
\multirow{3}*{{\textbf{BDSC-LP}}}
& 0.2468$\pm$0.0683 & 0.2043$\pm$0.0486 & 0.3862$\pm$0.0752 & 0.4202$\pm$0.1027 & 0.4619$\pm$0.0355 \\
& 0.4776$\pm$0.0504 & 0.5646$\pm$0.0822 & 0.6844$\pm$0.0817 & 0.7014$\pm$0.0825 & 0.7496$\pm$0.0546 \\
& 0.9029$\pm$0.0655 & 0.8611$\pm$0.0432 & 0.7798$\pm$0.0379 & 0.7409$\pm$0.0923 & 0.7164$\pm$0.0926 \\

\hline
\multirow{3}*{{DDSC+SF}}
& 0.2747$\pm$0.0581 & 0.3235$\pm$0.0481 & 0.4389$\pm$0.0291 & 0.4704$\pm$0.0640 & 0.5167$\pm$0.0679 \\
& 0.6050$\pm$0.0512 & \textbf{0.6702}$\pm$0.0376 & 0.7467$\pm$0.0171 & \textbf{0.7551}$\pm$0.0115 & 0.7866$\pm$0.0115 \\
& 0.8702$\pm$0.0541 & 0.8209$\pm$0.0650 & 0.7593$\pm$0.0720 & 0.7289$\pm$0.0805 & 0.7091$\pm$0.0273 \\

\hline
\multirow{3}*{{DDSC+TCP+GL}}
& 0.1515$\pm$0.0895 & 0.2186$\pm$0.0711 & 0.3421$\pm$0.0844 & 0.3858$\pm$0.0865 & 0.4187$\pm$0.0652 \\
& 0.4804$\pm$0.0472 & 0.5856$\pm$0.0333 & 0.6916$\pm$0.0278 & 0.7218$\pm$0.0217 & 0.7542$\pm$0.0481 \\
& 0.9110$\pm$0.0850 & 0.8674$\pm$0.0730 & 0.7839$\pm$0.0401 & 0.7579$\pm$0.0401 & 0.7289$\pm$0.0617 \\

\hline
\multirow{3}*{{DDSC}}
& 0.1268$\pm$0.0931 & 0.2063$\pm$0.0264 & 0.2621$\pm$0.0275 & 0.3261$\pm$0.0753 & 0.3711$\pm$0.0357 \\
& 0.4408$\pm$0.0258 & 0.5523$\pm$0.0590 & 0.6544$\pm$0.0417 & 0.6739$\pm$0.0271 & 0.7154$\pm$0.0568 \\
& 0.9271$\pm$0.1080 & 0.8867$\pm$0.0692 & 0.7923$\pm$0.0579 & 0.7631$\pm$0.0401 & 0.7287$\pm$0.0097 \\

\hline
\multirow{3}*{{FHMM}}
& 0.2843$\pm$0.0791 & 0.3331$\pm$0.0694 & 0.4277$\pm$0.0313 & 0.4641$\pm$0.0544 & 0.4801$\pm$0.0484 \\
& 0.5804$\pm$0.0631 & 0.6324$\pm$0.0632 & 0.7289$\pm$0.0396 & 0.7313$\pm$0.0235 & 0.7619$\pm$0.0135 \\
& 0.8729$\pm$0.0737 & 0.8134$\pm$0.0564 & 0.7338$\pm$0.0589 & \textbf{0.7128}$\pm$0.0576 & 0.6784$\pm$0.0160 \\

\hline
\end{tabular}
\caption{
Scalable evaluations of disaggregation methods on \textbf{synthetic} data for time periods varying from 50 to 1000 days.
For each size of data, 10-fold cross validation was applied and the mean$\pm$std of Avg. F-measure, Accuracy, NDE are reported, where each metric occupies one line.
Bold entries denote the best performance values.
}
\label{table:syn_scal}
\end{table}
Generally, the performance of all the methods increased with the data size from 50 to 1000 days, and there is a relatively large performance gap from 100 to 400 days.
BDSC-LP+SF and DDSC+SF respectively outperformed BDSC-LP and DDSC with respect to all the three metrics, and these gains can be attributed to the customizations of basis functions with the help of shape features.
Compared with DDSC, BDSC-LP achieved a general better performance, verifying that the Bayesian treatment of the model provided the advantage of estimating device's consumption from aggregated data.
Through comparing DDSC+TCP+GL with DDSC, the TCP and GL extensions played a small part in performance enhancement.
It also showed that FHMM achieved similar results to DDSC+SF.

\subsection{Performance Evaluations over Real World Data} \label{subsec:large-real}
The performance was evaluated using a 10-fold cross validation for each home: the water use events were divided for each home randomly into 10 approximately equal-sized groups; the above methods were trained on the combined data from 9 of these groups and then tested on the group of data  withheld; this was repeated withholding each of the 10 groups in turn and the average performance over all homes was reported.
\begin{table}[htb!] \footnotesize
\centering
\begin{tabular}{c |ccccc}
\hline
\multirow{1}* {
\diagbox[width=2.2cm]{\raisebox{-0.5pt}{\hspace*{0.0cm}{ Methods}}}{\raisebox{-1.0pt}{\hspace*{-1.0cm}{ Devices}}}}  
& \multirow{2}* {\textbf{Faucet}} & 
  \multirow{2}* {\textbf{Dishwasher}} & 
  \multirow{2}* {\textbf{Toilet}} & 
  \multirow{2}* {\textbf{Shower}} & 
  \multirow{2}* {\textbf{Clothes Washer}} \\
& & & & & \\

\hline
\multirow{3}*{{\textbf{BDSC-LP+SF}}}
& 0.4460$\pm$0.1104 & 0.1556$\pm$0.0472 & 0.5138$\pm$0.1010 & 0.6664$\pm$0.1162 & 0.5924$\pm$0.0681 \\
& 0.3989$\pm$0.0961 & 0.4610$\pm$0.0820 & 0.5975$\pm$0.0950 & 0.5511$\pm$0.0903 & 0.5595$\pm$0.0815 \\
& \textbf{0.4173}$\pm$0.0917 & 0.2270$\pm$0.0486 & \textbf{0.5459}$\pm$0.0701 & \textbf{0.5922}$\pm$0.0065 & \textbf{0.5749}$\pm$0.0717 \\

\hline
\multirow{3}*{{\textbf{BDSC-LP}}}
& 0.2853$\pm$0.0608 & 0.1188$\pm$0.0227 & 0.4761$\pm$0.0639 & 0.3644$\pm$0.0926 & 0.3515$\pm$0.0709 \\
& 0.5039$\pm$0.0592 & 0.3369$\pm$0.1082 & 0.2417$\pm$0.0382 & 0.5963$\pm$0.0596 & 0.4762$\pm$0.0726 \\
& 0.3631$\pm$0.0621 & 0.1743$\pm$0.0369 & 0.3189$\pm$0.0406 & 0.4444$\pm$0.0648 & 0.4030$\pm$0.0665 \\

\hline
\multirow{3}*{{DDSC+SF}}
& 0.4236$\pm$0.0736 & 0.1555$\pm$0.0130 & 0.4958$\pm$0.1072 & 0.6851$\pm$0.0980 & 0.6230$\pm$0.0790 \\
& 0.3268$\pm$0.1788 & 0.4728$\pm$0.0914 & 0.6092$\pm$0.0934 & 0.5185$\pm$0.1345 & 0.4847$\pm$0.0728 \\
& 0.3347$\pm$0.1208 & 0.2324$\pm$0.0018 & 0.5363$\pm$0.0385 & 0.5756$\pm$0.0432 & 0.5451$\pm$0.0760 \\

\hline
\multirow{3}*{{DDSC+TCP+GL}}
& 0.2686$\pm$0.0567 & 0.1217$\pm$0.0440 & 0.4995$\pm$0.0787 & 0.3753$\pm$0.0776 & 0.3887$\pm$0.0674 \\
& 0.5472$\pm$0.0738 & 0.3064$\pm$0.0335 & 0.2317$\pm$0.0414 & 0.5472$\pm$0.0738 & 0.3682$\pm$0.0973 \\
& 0.3570$\pm$0.0559 & 0.1722$\pm$0.0475 & 0.3142$\pm$0.0421 & 0.4370$\pm$0.0269 & 0.3750$\pm$0.0778 \\

\hline
\multirow{3}*{{DDSC}}
& 0.2614$\pm$0.0582 & 0.1152$\pm$0.0522 & 0.4603$\pm$0.1303 & 0.3613$\pm$0.1213 & 0.3572$\pm$0.0940 \\
& 0.4572$\pm$0.1351 & 0.2764$\pm$0.0456 & 0.2095$\pm$0.0788 & 0.6313$\pm$0.0718 & 0.4394$\pm$0.0568 \\
& 0.3187$\pm$0.0057 & 0.1557$\pm$0.0451 & 0.2759$\pm$0.0778 & 0.4563$\pm$0.1153 & 0.3890$\pm$0.0633 \\

\hline
\multirow{3}*{{FHMM}}
& 0.3626$\pm$0.1008 & 0.1953$\pm$0.0768 & 0.4256$\pm$0.0181 & 0.4808$\pm$0.1486 & 0.5619$\pm$0.1157 \\
& 0.4720$\pm$0.0543 & 0.4865$\pm$0.0687 & 0.6646$\pm$0.0711 & 0.2828$\pm$0.0593 & 0.4002$\pm$0.1122 \\
& 0.4001$\pm$0.0551 & \textbf{0.2720}$\pm$0.0850 & 0.5185$\pm$0.0342 & 0.3547$\pm$0.0869 & 0.4663$\pm$0.1148 \\

\hline
\end{tabular}
\caption{
	Disaggregation results on the real data:
	10-fold cross-validation was applied for each home, and the mean$\pm$std of Precision, Recall, F-measure are reported, where each metric occupies one line.
	The bold entries denote the best F-measure.
}
\label{table:homo_real_device}
\end{table}
\begin{figure}[htb!]
\centering
\includegraphics[width=0.8\textwidth]{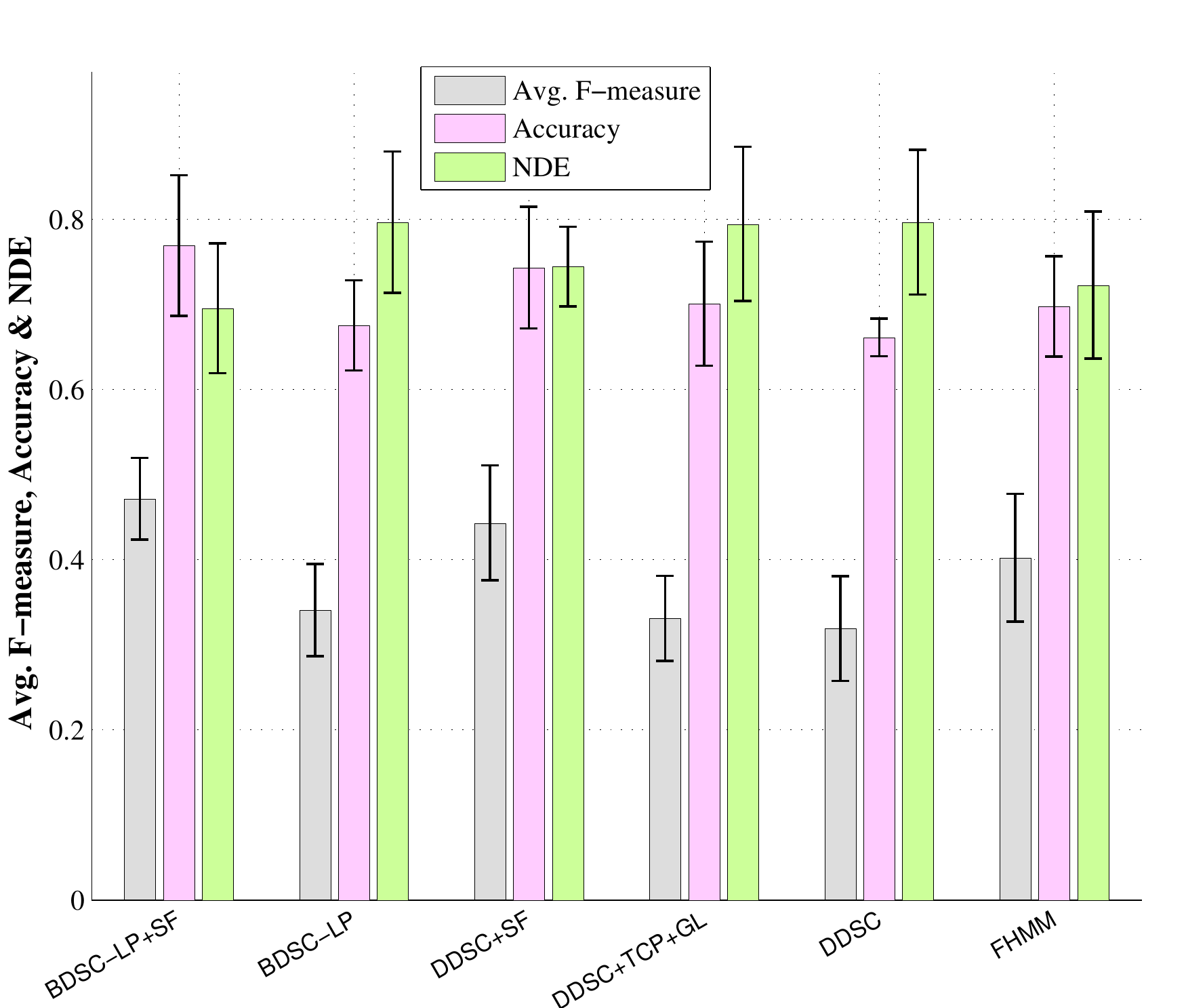}
\caption{Whole-home level performance evaluations on real data: 10-fold cross-validation was applied and the mean$\pm$std of Avg. F-measure, Accuracy and NDE are plotted as bars.}
\label{fig:homo_whole_performance}
\end{figure}
The device and whole home level evaluation results are shown in Table~\ref{table:homo_real_device} and Figure~\ref{fig:homo_whole_performance}, respectively.
BDSC-LP+SF provided better performance than BDSC-LP in terms of qualitative Precision, Recall and F-measure.
Meanwhile BDSC-LP+SF and DDSC+SF respectively outperformed BDSC-LP and DDSC.
These findings indicate that shape features were critical for improving performance.
The fact that BDSC-LP slightly outperformed DDSC showed that the Bayesian treatment of the sparse coding model was once again a better choice.
TCP and GL were of small significance in performance enhancement and the performance of DDSC+TCP+GL was a little better than that of DDSC.
FHMM could achieve acceptable results, which were similar to those produced DDSC+SF. 
At the whole home level, as expected, the values of Avg. F-measure, Accuracy and NDE achieved by BDSC-LP+SF were better than others.
Employing shape features enabled BSDC-LP+SF and DDSC+SF to outperform BSDC-LP and DDSC, respectively.
The Bayesian treatment of the sparse coding model allowed BDSC-LP to produce a slight better performance than BSC-L.
DDSC+TCP+GL also exhibited a slight better performance than DDSC, while FHMM had a similar performance to DDSC+SF.

\begin{figure}[t]
\centering
\includegraphics[width=0.85\textwidth]{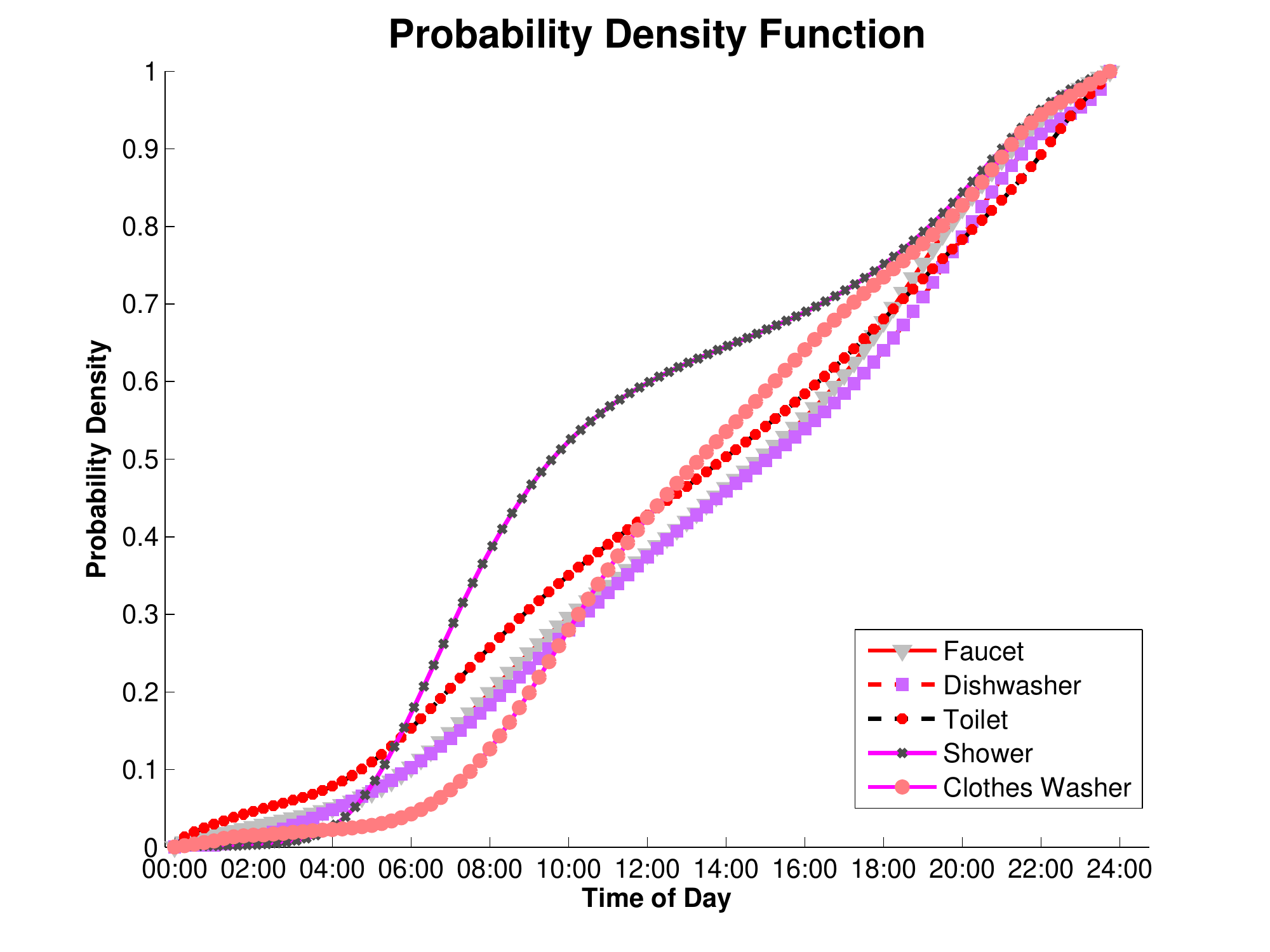}
\caption{CDFs of events' starting interval.}
\vspace{-0.0pc}
\label{fig:cdf_devices}
\end{figure}

\begin{figure}[t]
\centering
\includegraphics[width=0.85\textwidth]{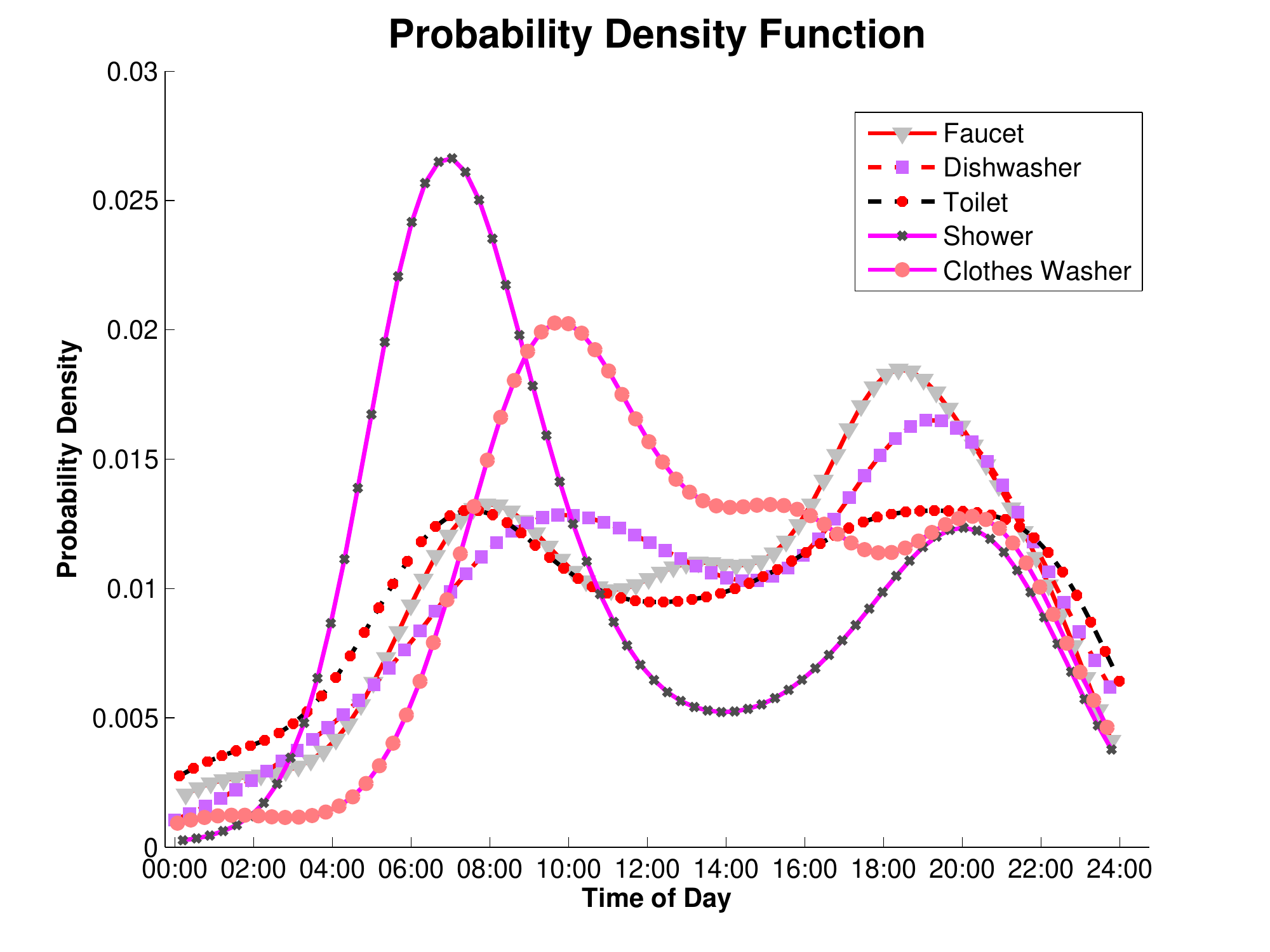}
\caption{PDFs of events' starting interval.}
\vspace{-0.0pc}
\label{fig:pdf_devices}
\end{figure}

\subsection{Discussion} \label{subsec:discussion}
The above experimental results demonstrate that the proposed models significantly outperformed the baselines for the task of water disaggregation.
The experimental results also verified the following observations.
\textit{1). \textbf{Utilization of domain knowledge.} }
The domain/prior knowledge suggests that the duration and consumption trend of water fixtures (corresponding to the human activities related to water consumption) are distinct across devices.
By formalizing and customizing the shape features to help learn discriminative dictionaries, BDSC-LP+SF and DDSC+SF respectively outperformed BDSC-LP and DDSC.
\textit{2). \textbf{Bayesian treatment of the discriminative sparse coding model.} }
The Bayesian treatment of the discriminative sparse coding model allows the model to be more flexible for the learning of dictionaries for enhancing the disaggregation performance.
The performance comparision results between BDSC-LP and DDSC in terms of Avg. F-measure, Accuracy and NDE validated that the Bayesian discriminative model could usually better results than the conventional discriminative model.

\section{Conclusions} \label{sec:conclusion}
This paper presents a shape features based Bayesian discriminative sparse coding model for low-sampling-rate water disaggregation.
Bayesian modeling of the discriminative sparse coding model can help for promoting the disaggregation performance.
Supported by an in-depth study of real-word consumption data, we propose the use of shape features to capture the changing characteristics of the data, followed by the application of basis smoothness to further increase new model's capacity to derive more sparse coefficients.
Gibbs sampling based methods are developed for model inference and parameter estimations.
Using both synthetic and real data sets, our experimental results showed that the proposed models significantly outperformed baselines at both the whole-home and device levels.

\appendix
\section{Additional Experimental Results}\label{app:C}


\subsection{CDFs and PDFs of Starting Intervals}\label{subsec:appC1}
Figure~\ref{fig:cdf_devices} shows the CDFs of starting interval for the five devices, while the corresponding PDFs (Probability Density Function) are shown in Figure~\ref{fig:pdf_devices}.

Checking the CDF and PDF of Faucet, we speculate that people use more Faucets before/after breakfast ($07:00\sim 09:00$) or before/after dinner ($17:30\sim 20:00$).
For Dishwasher, we observe that it happens more frequently at evening ($18:00\sim 20:00$), and indicate that people like to wash dishes after dinner.
With respect to Toilet, as expected, more Toilets happen before/after getting up ($07:00\sim 09:00$).
The patterns of Shower are the most distinctive: people take a Shower in morning ($06:00\sim 08:00$) or evening ($19:00\sim 21:00$).
Based on the observation of Clothes Washer, we find that morning (but not that obvious) is preferred by people for clothes washing.



\bibliographystyle{unsrt}
\bibliography{water}

%
%
%

%

%
%
%

%
%

\end{document}